\documentclass{article}

\usepackage[final]{neurips_2022}
\usepackage{subfigure}

\usepackage[table]{xcolor}
\usepackage{colortbl}
\usepackage{amsmath}
\usepackage{enumitem}
\usepackage{booktabs}
\usepackage{makecell}
\usepackage{algorithm}
\usepackage{amsfonts, amssymb}
\usepackage{algorithmic}
\usepackage{mathtools}
\usepackage{amsthm}
\usepackage[mathscr]{eucal}
\usepackage{bm}
\usepackage{graphicx}
\usepackage{multirow}
\usepackage{amsmath,lipsum}
\usepackage{amssymb} 
\usepackage{wrapfig}
\usepackage{ulem}

\usepackage{tabularx}
\usepackage{cuted}
\usepackage[utf8]{inputenc}
\usepackage{ragged2e}

\usepackage[english]{babel}
\newtheorem{theorem}{Theorem}
\newtheorem{corollary}{Corollary}[theorem]
\newtheorem{lemma}[theorem]{Lemma}

\usepackage{subfiles}
\usepackage{xr}
\definecolor{myblue}{HTML}{b2f0ff}
\definecolor{myblue2}{HTML}{cef5ff}
\definecolor{myblue3}{HTML}{e7faff}

\usepackage{color}

\usepackage{amsmath,amsfonts,bm}

\def\eqref#1{equation~\ref{#1}}

\def\1{\bm{1}}

\def\rvu{{\mathbf{i}}}

\def\rvu{{\mathbf{u}}}

\def\rvw{{\mathbf{w}}}
\def\rvx{{\mathbf{x}}}
\def\rvy{{\mathbf{y}}}
\def\rvz{{\mathbf{z}}}

\def\mJ{{\bm{J}}}

\DeclareMathAlphabet{\mathsfit}{\encodingdefault}{\sfdefault}{m}{sl}
\SetMathAlphabet{\mathsfit}{bold}{\encodingdefault}{\sfdefault}{bx}{n}

\usepackage[utf8]{inputenc} 
\usepackage[T1]{fontenc}    
\usepackage{hyperref}       
\usepackage{url}            
\usepackage{booktabs}       
\usepackage{amsfonts}       
\usepackage{nicefrac}       
\usepackage{microtype}      
\usepackage{xcolor}         

\title{Subspace Identification for Multi-Source Domain Adaptation}

\author{%
Zijian Li$^{2,3}$, Ruichu Cai$^{2}$\thanks{Corresponding authors.}, Guangyi Chen$^{3,1}$, Boyang Sun$^{3}$, Zhifeng Hao$^{4}$, Kun Zhang$^{3,1*}$~\\$^1$ Carnegie Mellon University \\ $^2$ School of Computer Science, Guangdong University of Technology \\$^3$ Mohamed bin Zayed University of Artificial Intelligence\\$^4$ Shantou University
}

\begin{document}

\maketitle

\begin{abstract}

Multi-source domain adaptation (MSDA) methods aim to transfer knowledge from multiple labeled source domains to an unlabeled target domain. 
Although current methods achieve target joint distribution identifiability by enforcing minimal changes across domains, they often necessitate stringent conditions, such as an adequate number of domains, monotonic transformation of latent variables, and invariant label distributions.
These requirements are challenging to satisfy in real-world applications. To mitigate the need for these strict assumptions, we propose a subspace identification theory that guarantees the disentanglement of domain-invariant and domain-specific variables under less restrictive constraints regarding domain numbers and transformation properties, thereby facilitating domain adaptation by minimizing the impact of domain shifts on invariant variables. Based on this theory, we develop a Subspace Identification Guarantee (SIG) model that leverages variational inference. Furthermore, the SIG model incorporates class-aware conditional alignment to accommodate target shifts where label distributions change with the domains. Experimental results demonstrate that our SIG model outperforms existing MSDA techniques on various benchmark datasets, highlighting its effectiveness in real-world applications. \footnote{https://github.com/jozerozero/Subspace\_Identification}

\end{abstract}

\section{Introduction}


Multi-Source Domain Adaptation (MSDA) is a method of transferring knowledge from multiple labeled source domains to an unlabeled target domain, to address the challenge of domain shift between the training data and the test environment. Mathematically, in the context of MSDA, we assume the existence of $M$ source domains $\{\mathcal{S}_1,\mathcal{S}_2, ..., \mathcal{S}_M\}$ and a single target domain $\mathcal{T}$. 
For each source domain $\mathcal{S}_{i}$, data are drawn from a distinct distribution, represented as $p_{\rvx,\rvy|\rvu_{\mathcal{S}_{i}}}$, where the variables $\rvx, \rvy, \rvu$ correspond to features, labels, and domain indices, respectively.  In a similar manner, the distribution within the target domain $\mathcal{T}$ is given by $p_{\rvx, \rvy| \rvu_{\mathcal{T}}}$. In the source domains, we have access to $m_i$ annotated feature-label pairs of each domain, denoted by $(\rvx^{\mathcal{S}_i},\rvy^{\mathcal{S}_i})=(\rvx^{\mathcal{S}_i}_k,y^{\mathcal{S}i}_k)^{m_i}_{k=1}$, while in the target domain, only $m_{\mathcal{T}}$ unannotated features are observed, represented as  $(\rvx^{(\mathcal{T})})=(\rvx^{\mathcal{T}}_k)^{m_{\mathcal{T}}}_{k=1}$.
The primary goal of MSDA is to effectively leverage these labeled source data and unlabeled target data to identify the target joint distribution $p_{\rvx, \rvy| \rvu_{\mathcal{T}}}$.

However, identifying the target joint distribution of $\rvx,\rvy | \rvu_{\mathcal{T}}$ using only  $\rvx | \rvu_{\mathcal{T}}$ as observations present a significant challenge, since the possible mappings from $p_{\rvx, \rvy| \rvu_{\mathcal{T}}}$ to $p_{\rvx|\rvu_{\mathcal{T}}}$ are infinite when no extra constraints are given. 
To solve this problem, some assumptions have been proposed to constrain the domain shift, such as covariate shift~\cite{pan2010survey}, target shift~\cite{tachet2020domain,azizzadenesheli2019regularized}, and conditional shift~\cite{cai2019learning,stojanov2021domain}.  
For example, the most conventional covariate shift assumption posits that $p_{\rvy|\rvx}$ is fixed across different domains while $p_{\rvx}$ varies.  Under this assumption, researchers can employ techniques such as  importance reweighting~\cite{shimodaira2000improving}, invariant representation learning~\cite{ganin2015unsupervised}, or cycle consistency~\cite{hoffman2018cycada} for distribution alignment. 
Additionally, target shift assumes that $p_{\rvx|\rvy}$ is fixed while the label distribution $p_{\rvy}$ changes, whereas conditional shift is represented by a fixed $p_{\rvy}$ and a varying $p_{\rvx|\rvy}$. More generally, the minimal change principle has been proposed, which not only unifies the aforementioned assumptions but also enables the theoretical guarantee of the identifiability of the target joint distribution. Specifically, it assumes that $p_{\rvy|\rvu}$ and $p_{\rvx|\rvy,\rvu}$ change independently and the change of $p_{\rvx|\rvy,\rvu}$ is minimal. 
Please refer to Appendix \ref{app:related_works} for further discussion of related works, including domain adaptation and identification.

Although current methods demonstrate the identifiability of the target joint distribution through the minimal change principle, they often impose strict conditions on the data generation process and the number of domains, limiting their practical applicability. For instance, iMSDA~\cite{kong2022partial} presents the component-wise identification of the domain-changed latent variables, subsequently identifying target joint distribution by modeling a data generation process with variational inference. However, this identification requires the following conditions. 
First, a sufficient number of auxiliary variables is employed for the component-wise theoretical guarantees, meaning that when the dimension of latent variables is $n$, a total of $2n+1$ domains are needed. Second, in order to identify high-level invariant latent variables, a component-wise monotonic function between latent variables must be assumed. Third, these methods implicitly assume that label distribution remains stable across domains, despite the prevalence of target shift in real-world scenarios. These conditions are often too restrictive to be met in practice, highlighting the need for a more general approach to identifying latent variables across a wider range of domain shifts.


In an effort to alleviate the need for such strict assumptions, we present a subspace identification theory in this paper that guarantees the disentanglement of domain-invariant and domain-specific variables under more relaxed constraints concerning the number of domains and transformation properties. In contrast to component-wise identification, our subspace identification method demands fewer auxiliary variables (i.e., when the dimension of latent variables is $n$, only $n+1$ domains are required). Additionally, we design a more general data generation process that accounts for target shift and does not necessitate monotonic transformation between latent variables. In this process, we categorize latent variables into four groups based on whether they are influenced by domain index or label.
Building on the theory and causal generation process, we develop a Subspace Identification Guarantee (SIG) model that employs variational inference to identify latent variables. A class-aware condition alignment is incorporated to mitigate the impact of target shift, ensuring the update of the most confident cluster embedding. Our approach is validated through a simulation experiment for subspace identification evaluation and four widely-used public domain adaptation benchmarks for application evaluation. The impressive performance that outperforms state-of-the-art methods demonstrates the effectiveness of our method.

\section{Identifying Target Joint Distribution with Data Generation Process}

\subsection{Data Generation Process}
\label{sec:data_gen}
\begin{wrapfigure}{r}{5cm}
 \vspace{-19pt}
	\centering
	\includegraphics[width=0.35\columnwidth]{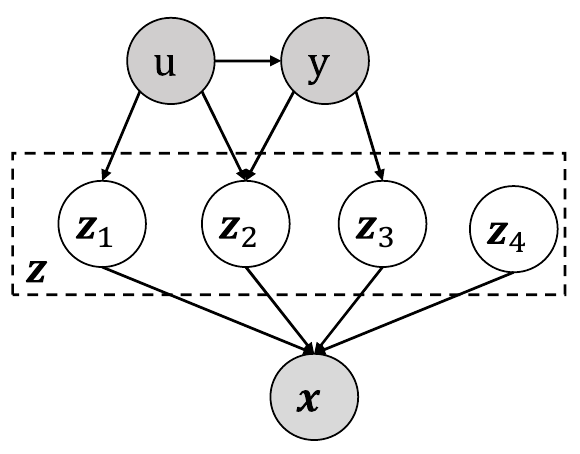}
	\caption{Data generation process, where the gray  the white nodes denote the observed and latent variables, respectively.}
 \medskip\medskip
	\label{fig:data}
 \vspace{-15pt}
\end{wrapfigure}

%

We begin with introducing the data generation process. As shown in Figure \ref{fig:data}, the observed data $\bm{x}\in \mathcal{X}$ are generated by latent variables $\bm{\rvz}\in \mathcal{Z} \subseteq \mathbb{R}^n$. Sequentially, we divide the latent variables $\bm{\rvz}$ into the four parts, i.e. $\bm{\rvz}=\{\bm{\rvz}_{1}, \bm{\rvz}_{2}, \bm{\rvz}_{3}, \bm{\rvz}_{4}\} \in \mathcal{Z} \subseteq \mathbb{R}^n$,  
which are shown as follows.
\begin{itemize}[leftmargin=*,itemsep=-3pt]
\vspace{-5pt}
    \item domain-specific and label-irrelevant variables $\bm{\rvz}_{1}\in \mathbb{R}^{n_1}$.
    \item domain-specific but label-relevant variables $\bm{\rvz}_{2}\in \mathbb{R}^{n_2}$.
    \item domain-invariant and label-relevant variables $\bm{\rvz}_{3} \in \mathbb{R}^{n_3}$.
    \item domain-invariant but label-irrelevant variables $\bm{\rvz}_{4} \in \mathbb{R}^{n_4}$.
\vspace{-5pt}
\end{itemize}


To better understand these latent variables, we provide some examples in DomainNet datasets. First, $\bm{\rvz}_{1} \in \mathbb{R}^{n_1}$ denotes the styles of the images like ``infograph'' and ``sketch'', which are irrelevant to labels. $\bm{\rvz}_{2} \in \mathbb{R}^{n_2}$ denotes the latent variables that can be the texture information relevant to domains and labels. For example, the samples of ``clock'' and ``telephone'' contain some digits, and these digits in these samples are a special texture, which can be used for classification and be influenced by different styles, such as ``infograph''. $\bm{\rvz}_{3} \in \mathbb{R}^{n_3}$ denotes the latent variables that are only relevant to the labels. For example, in the DomainNet dataset, it can be interpreted as the meaning of different classes like ``Bicycle'' or ``Teapot''.  Finally, $\bm{\rvz}_{4}\in \mathbb{R}^{n_4}$ denotes the label-irrelevant latent variables. For example, $\bm{\rvz}_{4}$ can be interpreted as the background that is invariant to domains and labels.


Based on the definitions of these latent variables, we let the observed data be generated from $\bm{\rvz}$ through an invertible and smooth mixing function $g:\mathcal{Z} \rightarrow \mathcal{X}$. 
Due to the target shift, we further consider that the $p_{\rvy}$ is influenced by $\rvu$, i.e. $\rvu \rightarrow \rvy$.

Compared with the existing data generation process like \cite{kong2022partial}, the proposed data generation process is different in three folds. First, $p_{\rvu}$ is independent of $p_{\rvy}$ in the iMSDA \cite{kong2022partial}, so the target shift is not taken into account. Second, the data generation process of iMSDA requires an invertible and monotonic function between latent variables for component-wise identification, which is too strict to be met in practice. Third, to provide a more general way to depict the real-world data, our data generation process introduces the domain-specific but label-relevant latent variables $\bm{\rvz}_{s_2}$ and the domain-invariant but label-irrelevant variables $\rvz_4$.

\subsection{Identifying the Target Joint Distribution}

In this part, we show how to identify the target joint distribution  $p_{\rvx,\rvy|\rvu_{\mathcal{T}}}$  with the help of marginal distribution. By introducing the latent variables and combining the proposed data generation process, we can obtain the following derivation.
\begin{equation}
\begin{split}    p_{\rvx,\rvy|\rvu_{\mathcal{T}}}&=\int_{\bm{\rvz}_{1}}\int_{\rvz_{2}}\int_{\rvz_{3}}\int_{\rvz_{4}}p_{\rvx,\rvy,\bm{\rvz}_{1},\bm{\rvz}_{2},\bm{\rvz}_{3},\bm{\rvz}_{4}|\rvu_{\mathcal{T}}}d\bm{\rvz}_{1}d\bm{\rvz}_{2}d\bm{\rvz}_{3}d\bm{\rvz}_{4}\\&
=\int_{\bm{\rvz}_{1}}\int_{\rvz_{2}}\int_{\rvz_{3}}\int_{\rvz_{4}}p_{\rvx,\bm{\rvz}_{1},\bm{\rvz}_{2},\bm{\rvz}_{3},\bm{\rvz}_{4}|\rvy,\rvu_{\mathcal{T}}}\cdot p_{\rvy|\rvu_{\mathcal{T}}}d\bm{\rvz}_{1}d\bm{\rvz}_{2}d\bm{\rvz}_{3}d\bm{\rvz}_{4}
\\&=\int_{\bm{\rvz}_{1}}\int_{\rvz_{2}}\int_{\rvz_{3}}\int_{\rvz_{4}}p_{\rvx|\bm{\rvz}_{1},\bm{\rvz}_{2},\bm{\rvz}_{3},\bm{\rvz}_{4}}\cdot p_{\bm{\rvz}_{1},\bm{\rvz}_{2},\bm{\rvz}_{3},\bm{\rvz}_{4}|\rvy,\rvu_{\mathcal{T }}} \cdot p_{\rvy|\rvu_{\mathcal{T}}} d\bm{\rvz}_{1}d\bm{\rvz}_{2}d\bm{\rvz}_{3}d\bm{\rvz}_{4}.
\end{split}
\label{equ:tgt_joint_dist_paper}
\vspace{-10pt}
\end{equation}

According to the derivation in Equation (\ref{equ:tgt_joint_dist_paper}), we can identify the target joint distribution by modeling three distributions. First, we need to model $p_{\rvx|\bm{\rvz}_{1},\bm{\rvz}_{2},\bm{\rvz}_{3},\bm{\rvz}_{4}}$, implying that we need to model the conditional distribution of observed data give latent variables, which coincides with a generative model for observed data. Second, we need to estimate the label pseudo distribution of target domain $p_{\rvy|\rvu_{\mathcal{T}}}$. Third, we need to model $p_{\bm{\rvz}_{1},\bm{\rvz}_{2},\bm{\rvz}_{3},\bm{\rvz}_{4}|\rvy,\rvu_{\mathcal{T}}}$ meaning that the latent variables should be identified with auxiliary variables $\rvu,\rvy$ under theoretical guarantees. In the next section, we will introduce how to identify these latent variables with subspace identification block-wise identification results.

\section{Subspace Identifiability for Latent Variables}

\begin{wrapfigure}{r}{3cm}
 \vspace{-10pt}
	\centering
	\includegraphics[width=0.18\columnwidth]{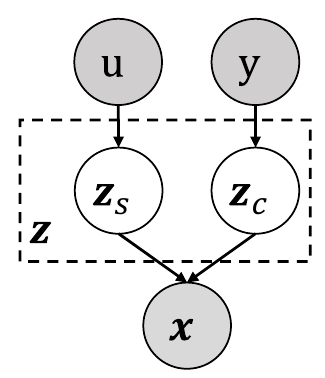}
	\caption{A simple data generalization process for introducing subspace identification.}
 \medskip
	\label{fig:data2_paper}
 \vspace{-19pt}
\end{wrapfigure}
In this section, we provide how to identify the latent variables in Figure \ref{fig:data}. In detail, we first prove that $\rvz_2$ is subspace identifiable and $\rvz_1,\rvz_3$ can be reconstructed from the estimated $\hat{\rvz}_1,\hat{\rvz}_2, \hat{\rvz}_3$. Then we further prove that $\rvz_4$ is block-wise identifiable. 

To clearly introduce the subspace identification theory, we employ a simple data generation process \cite{cai2019learning} as shown in Figure \ref{fig:data2_paper}. In this data generation process, $\rvz_s \in \mathbb{R}^{n_s}$ and $\rvz_c \in \mathbb{R}^{n_c}$ denote the domain-specific and domain-invariant latent variables, respectively. For convenient, we let $\rvz=\{\rvz_s, \rvz_c\}, n=n_s+n_c$. Moreover, we assume $\rvz_s=(z_i)_{i=1}^{n_s}$ and $\rvz_c=(z_i)_{i=n_s+1}^n$. And $\{\rvu,\rvy,\rvx\}$ denote the domain index, labels, and observed data, respectively. And we further let the observed data be generated from $\bm{\rvz}$ through an invertible and smooth mixing function $g:\mathcal{Z} \rightarrow \mathcal{X}$. The subspace identification of $\rvz_s$ means that for each ground-truth $z_{s,i}$, there exits $\hat{\rvz}_s$ and an invertible function $h_{i}:\mathbb{R}^{n}\rightarrow \mathbb{R}$, such that $z_{s,i}=h_{i}(\hat{\rvz}_s)$.


\begin{theorem}
\label{the1_paper}
(\textbf{Subspace Identification of $\rvz_s$.}) We follow the data generation process in Figure \ref{fig:data2_paper} and make the following assumptions:
\begin{itemize}[leftmargin=*,itemsep=-3pt]
    \item A1 (\underline{Smooth and Positive Density}): The probability density function of latent variables is smooth and positive, i.e., $p_{\rvz|\rvu}>0$ over $\mathcal{Z}$ and $\mathcal{U}$.
    \item A2 (\underline{Conditional independent}): Conditioned on $\rvu$, each $z_i$ is independent of any other $z_j$ for $i,j \in \{1,\cdots,n\}, i\neq j$, i.e. $\log p_{\rvz|\rvu}(\rvz|\rvu)=\sum_i^n q_i(z_{i},\rvu)$ where $q_i(z_{i},\rvu)$ is the log density of the conditional distribution, i.e., $q_i:\log p_{z_{i}|\rvu}$.
    \item A3 (\underline{Linear independence}): For any $\rvz_s\in \mathcal{Z}_s\subseteq \mathbb{R}^{n_s}$, there exist $n_s+1$ values of $\rvu$, i.e., $\rvu_j$ with $j=0,1,\cdots,n_s$, such that these $n_s$ vectors $\bm{\rvw}(\rvz,\rvu_j)-\bm{\rvw}(\rvz,\rvu_0)$ with $j=1,\cdots,n_s$ are linearly independent, where vector $\bm{\rvw}(\rvz,\rvu_j)$ is defined as follows:
    \begin{equation}
    \small
        \bm{\rvw}(\rvz,\rvu)=\left(\frac{\partial q_1(z_{1},\rvu)}{\partial z_{1}},\cdots, \frac{\partial q_i(z_{i},\rvu)}{\partial z_{i}},\cdots \frac{\partial q_{n_s}(z_{n_s},\rvu)}{\partial z_{n_s}}\right),
    \end{equation}
\end{itemize}
By modeling the aforementioned data generation process, $\rvz_s$ is subspace identifiable.
\end{theorem}
\textbf{Proof sketch.} First, we construct an invertible transformation $h$ between the ground-truth $\rvz$ and estimated $\hat{\rvz}$. Sequentially, we leverage the variance of different domains to construct a full-rank linear system, where the only solution of $\frac{\partial \rvz_s}{\partial \hat{\rvz}_c}$ is zero. Since the Jacobian of $h$ is invertible, for each $z_{s,i}, i\in\{1,\cdots,n_s\}$, there exists a $h_i$ such that  $z_{s,i}=h_i(\hat{\rvz})$ and $\rvz_s$ is subspace identifiable. 

The proof can be found in Appendix \ref{app:proof_subspace}. The first two assumptions are standard in the component-wise identification of existing nonlinear ICA \cite{kong2022partial,khemakhem2020variational}. The third Assumption means that $p_{\rvz|\rvu}$ should vary sufficiently over $n+1$ domains. Compared to component-wise identification, which necessitates $2n+1$ domains and is likely challenging to fulfill, subspace identification can yield equivalent results in terms of identifying the ground-truth latent variables with only $n+1$ domains. Therefore, subspace identification benefits from a more relaxed assumption.


Based on the theoretical results of subspace identification, we show that the ground-truth $\rvz_1, \rvz_2$ and $\rvz_3$ be reconstructed from the estimated $\hat{\rvz}_1, \hat{\rvz}_2$ and $\hat{\rvz}_3$. For ease of exposition, we assume that $\rvz_1,\rvz_2,\rvz_3$, and $\rvz_4$ correspond to components in $\rvz$ with indices $\{1,\cdots,n_1\},\{n_1+1,\cdots,n_1+n_2\},\{n_1+n_2+1,\cdots,n_1+n_2+n_3\}$, and $\{n_1+n_2+n_3+1, \cdots, n\}$, respectively. 

\begin{corollary}
We follow the data generation in Section \ref{sec:data_gen}, and make the following assumptions which are similar to A1-A3:
\item A4 (\underline{Smooth and Positive Density}): The probability density function of latent variables is smooth and positive, i.e., $p_{\rvz|\rvu, \rvy}>0$ over $\mathcal{Z}$, $\mathcal{U}$, and $\mathcal{Y}$.
\item A5 (\underline{Conditional independent}): Conditioned on $\rvu$ and $\rvy$, each $z_i$ is independent of any other $z_j$ for $i,j \in \{1,\cdots,n\}, i\neq j$, i.e. $\log p_{\rvz|\rvu,\rvy}(\rvz|\rvu,\rvy)=\sum_i^n q_i(z_{i},\rvu,\rvy)$ where $q_i(z_{i},\rvu,\rvy)$ is the log density of the conditional distribution, i.e., $q_i:\log p_{z_{i}|\rvu,\rvy }$.
\item A6 (\underline{Linear independence}): For any $\rvz\in\mathcal{Z}\subseteq\mathbb{R}^n$, there exists $n_1+n_2+n_3+1$ combination of $(\rvu,\rvy)$, i.e. $j=1,\cdots, U$ and $ c=1, \cdots, C$ and $ U\times C - 1=n_1+n_2+n_3$, where $U$ and $C$ denote the number of domains and the number of labels. such that these $n_1+n_2+n_3$ vectors $\bm{\rvw}(\rvz,\rvu_j,\rvy_c)-\bm{\rvw}(\rvz,\rvu_0,\rvy_0)$ are linearly independent, where $\bm{\rvw}(\rvz,\rvu_j,\rvy_c)$ is defined as follows:
\begin{equation}
\small
    \bm{\rvw}(\rvz,\rvu,\rvy)=\left(\frac{\partial q_1(z_{1},\rvu,\rvy)}{\partial z_{1}},\cdots, \frac{\partial q_i(z_{i},\rvu,\rvy)}{\partial z_{i}},\cdots \frac{\partial q_n(z_{n},\rvu,\rvy)}{\partial z_{n}}\right).
\end{equation}
By modeling the data generation process in Section \ref{sec:data_gen}, $\rvz_2$ is subspace identifiable, and $\rvz_1, \rvz_3$ can be reconstructed from $\hat{\rvz}_1, \hat{\rvz}_2$ and $\hat{\rvz}_2, \hat{\rvz}_3$, respectively.
\end{corollary}
\textbf{Proof sketch.} The detailed proof can be found in Appendix \ref{corollary}. First, we construct an invertible transformation $h$ to bridge the relation between the ground-truth $\rvz$ and the estimated $\hat{\rvz}$. Then, we repeatedly use Theorem \ref{the1_paper} three times by considering the changing of labels and domains. Hence, we find that the values of some blocks of the Jacobian of $h$ are zero. Finally, the Jacobian of $h$ can be formalized as Equation (\ref{equ:jacobian2_paper}).
\begin{wraptable}{r}{8cm}
\vspace{-5pt}
\begin{equation}
\small
\renewcommand{\arraystretch}{1.5}
\label{equ:jacobian2_paper}
\begin{gathered}
    \mJ_{h}=\begin{bmatrix}
    \begin{array}{c|c|c|c}
        \mJ_{h}^{1,1} & \mJ_{h}^{1,2}& \mJ_{h}^{1,3}=0& \mJ_{h}^{1,4}=0 \\ \cline{1-4}
        \mJ_{h}^{2,1}=0 & \mJ_{h}^{2,2}& \mJ_{h}^{2,3}=0& \mJ_{h}^{2,4}=0 \\ \cline{1-4}
        \mJ_{h}^{3,1}=0 & \mJ_{h}^{3,2}& \mJ_{h}^{3,3}& \mJ_{h}^{3,4}=0 \\ \cline{1-4}
        \mJ_{h}^{4,1} & \mJ_{h}^{4,2}& \mJ_{h}^{4,3}& \mJ_{h}^{4,4} \\
    \end{array}
    \end{bmatrix},
\end{gathered}
\end{equation}
\vspace{-8pt}
\end{wraptable}
where $\mJ_h$ denotes the Jacobian of $h$ and $\mJ^{ij}_h:=\frac{\partial \rvz_i}{\partial\hat{\rvz}_j}$ and $i,j\in\{1,2,3,4\}$. Since  $h(\cdot)$ is invertible, $\mJ_{h}$ is a full-rank matrix. Therefore, for each $z_{2,i}, i\in\{n_1+1,\cdots,n_1+n_2\}$, there exists a $h_{2,i}$ such that  $z_{2,i}=h_i(\hat{\rvz}_2)$. Moreover, for each $z_{1,i}, i\in\{1,\cdots,n_1+1\}$, there exists a $h_{1,i}$ such that  $z_{1,i}=h_{1,i}(\hat{\rvz}_1,\hat{\rvz}_2)$. And for each $z_{3,i}, i\in\{n_1+n_2+1,\cdots,n_1+n_2+n_3\}$, there exists a $h_{3,i}$ such that  $z_{3,i}=h_{3,i}(\hat{\rvz}_2,\hat{\rvz}_3)$. Then we prove that $\rvz_4$ is block-wise identifiable, which means that there exists an invertible function $h_4, s.t. \rvz_4=h_4(\hat{\rvz}_4)$.


\begin{lemma}
\label{lamma}\cite{kong2022partial}
Following the data generation process in Section \ref{sec:data_gen} and the assumptions A4-A6 in Theorem \ref{the1}, we further make the following assumption:
\begin{itemize}[leftmargin=*,itemsep=-3pt]
    \item A7 (\underline{Domain Variability}: For any set $A_{\rvz} \subseteq \mathcal{Z}$) with the following two properties: 1) $A_{\rvz}$ has nonzero probability measure, i.e. $\mathbb{P}[\rvz\in A_{\rvz}|\{\rvu=\rvu', \rvy=\rvy'\}]>0$ for any $\rvu'\in \mathcal{U}$ and $\rvy'\in \mathcal{Y}$. 2) $A_{\rvz}$ cannot be expressed as $B_{\rvz_4} \times \mathcal{Z}_1\times \mathcal{Z}_2\times \mathcal{Z}_3$ for any $B_{\rvz_4}\subset \mathcal{Z}_4$.
\end{itemize}
$\exists \rvu_1, \rvu_2 \in \mathcal{U}$ and $\rvy_1,\rvy_2 \in \mathcal{Y}$, such that $\int_{\rvz\in \mathcal{A}_{\rvz}}p_{\rvz|\rvu_1,\rvy_1}d\rvz \neq \int_{\rvz\in \mathcal{A}_{\rvz}}p_{\rvz|\rvu_2,\rvy_2}d\rvz$. By modeling the data generation process in Section 2.1, the $\rvz_4$ is block-wise identifiable. 
\end{lemma}

The proof can be found in Appendix \ref{app:proof_block}. Lemma \ref{lamma} shows that $\rvz_4$ can be block-wise identifiable when the $p_{\rvz|\rvu}$ changes sufficiently across domains. 

In summary, we can obtain the estimated latent variables $\hat{\rvz}$ with the help of subspace identification and block-wise identification.

\section{Subspace Identification Guarantee Model}
Based on the theoretical results, we proposed the Subspace Identification Guaranteed model (SIG) as shown in Figure \ref{fig:sig_model}, which contains a variational-inference-based neural architecture to model the marginal distribution and a class-aware conditional alignment to mitigate the target shift. 
\begin{figure}[t]
  \centering
	\includegraphics[width=0.85\columnwidth]{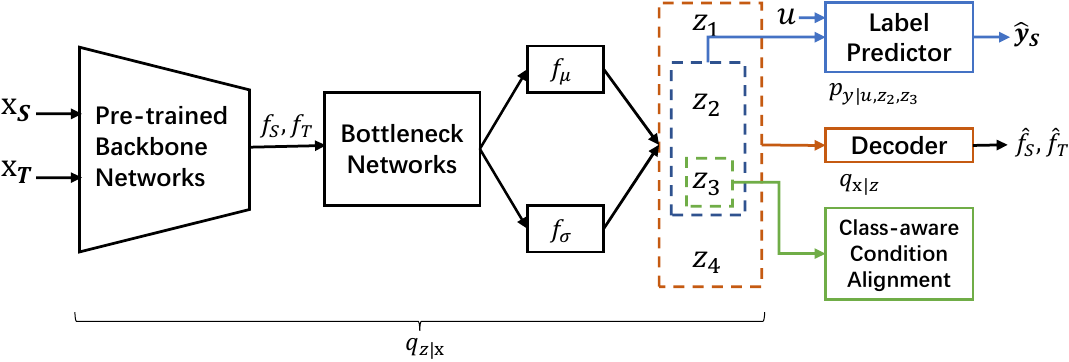}
  \caption{The framework of the Subspace Identification Guarantee model. The pre-trained backbone networks are used to extract the feature $f$ from observed data. The bottleneck and $f_{\mu},f_{\sigma}$ are used to generate $\bm{\rvz}$ with a  reparameterization trick. Label predictor takes $\rvz_2,\rvz_3$, and $\rvu$ as input to model $p_{\rvy|\rvu,\rvz_2,\rvz_3}$. The decoder is used to model the marginal distribution. Finally, $\rvz_2$ is used for class-aware conditional alignment.} 
  \label{fig:sig_model} 
\end{figure}

\subsection{Variational-Inference-based Neural Architecture} \label{sec:vae}

According to the data generation process in Figure \ref{fig:data}, we first derive the evidence lower bound (ELBO) in Equation (\ref{equ:elbo}).
\begin{equation}
\small
\begin{split}
    {ELBO}=&\mathbb{E}_{q_{\bm{\rvz}|\rvx}(\bm{\rvz}|\rvx)}\ln{p_{\rvx|\bm{\rvz}}}(\rvx|\bm{\rvz}) + \mathbb{E}_{q_{\bm{\rvz}|\rvx}(\bm{\rvz}|\rvx)}\ln{p_{\rvy|\rvu,\bm{\rvz}_{2},\bm{\rvz}_{3}}(\rvy|\rvu,\bm{\rvz}_{2},\bm{\rvz}_{3})} \\&+ \mathbb{E}_{q_{\bm{\rvz}|\rvx}(\bm{\rvz}|\rvx)}\ln{p_{\rvu|\bm{\rvz}}}(\rvu|\bm{\rvz}) - D_{KL}(q_{\bm{\rvz}|\rvx}(\bm{\rvz}|\rvx)||p_{\bm{\rvz}}(\rvz)).
\end{split}
\label{equ:elbo}
\vspace{-10pt}
\end{equation}

Since the reconstruction of $\rvu$ is not the optimization goal, we remove the reconstruction of $\rvu$ and we rewrite Equation (\ref{equ:elbo}) as the objective function in Equation (\ref{equ:loss_func}).
\begin{equation}
\small
\label{equ:loss_func}
\begin{split}
    \mathcal{L}_{elbo} &= \mathcal{L}_{vae} + \mathcal{L}_y \\
    \mathcal{L}_{vae} &= -\mathbb{E}_{q_{\bm{\rvz}|\rvx}(\bm{\rvz}|\rvx)}\ln{p_{\rvx|\bm{\rvz}}}(\rvx|\bm{\rvz}) + D_{KL}(q_{\bm{\rvz}|\rvx}(\bm{\rvz}|\rvx)||p_{\bm{\rvz}}(\rvz)) \\
    \mathcal{L}_y &= -\mathbb{E}_{q_{\bm{\rvz}|\rvx}(\bm{\rvz}|\rvx)}\ln{p_{\rvy|\rvu,\bm{\rvz}_{2},\bm{\rvz}_{3}}(\rvy|\rvu,\bm{\rvz}_{2},\bm{\rvz}_{3})}.
\end{split}
\end{equation}

To minimize the pairwise class confusion, we further employ the minimum class confusion \cite{jin2020minimum} into the classification loss $\mathcal{L}_y$.
According to the objective function in Equation (\ref{equ:loss_func}), we illustrate how to implement the SIG model in Figure \ref{fig:sig_model}. 
First, we take the observed data $\rvx_{S_i}$ and $\rvx_{T}$ from the source domains and the target domain as the inputs of the pre-trained backbone networks like ResNet50 and extract the feature $f_{\mathcal{S}_i}$ and $f_{\mathcal{T}}$. Sequentially, we employ an MLP-based encoder, which contains bottleneck networks, $f_{\mu}$ and $f_{\sigma}$, to extract the latent variables $\rvz$. Then, we take $\rvz$ to reconstruct the pre-trained features via an MLP-based decoder to estimate the marginal distribution $p_{\rvx|\rvu}$. Finally, we take the $\rvz_2, \rvz_3$, and the domain embedding to predict the source label to estimate $p_{\rvy|\rvu,\rvz_2,\rvz_3}$.

\vspace{-10pt}
\subsection{Class-aware Conditional Alignment}

To estimate the target label distribution $p_{\rvy|\rvu_{\mathcal{T}}}$ and mitigate the influence of target shift, we propose the class-aware conditional alignment to automatically adjust the conditional alignment loss for each sample. Formally, it can be written as 
\begin{equation}
\begin{split}
\mathcal{L}_{a} = \frac{1}{C} \sum_{i=1}^C w^{(i)}\cdot p_{{\hat{\rvy}}^{(i)}}||\hat{\rvz}_{3, \mathcal{S}}^{(i)} - \hat{\rvz}_{3, \mathcal{T}}^{(i)}||_2, \quad
    w^{(i)} = 1 + \exp{(-H(p_{\hat{\rvy}^{(i)}}))},
\end{split}
\label{equ:align3}
\end{equation}
where $C$ denotes the class number; $\hat{\rvz}_{3, \mathcal{S}}^{(i)}$ and $\hat{\rvz}_{3, \mathcal{T}}^{(i)}$ denote the latent variables of $i$th class from source and target domain, respectively; $w^{(i)}$ denotes the prediction uncertainty of each class in the target domain; $p_{\hat{\rvy}^{(i)}}$ denotes the estimated label probability density of $i$th class; $H$ denotes the entropy. 


The aforementioned class-aware conditional alignment is based on the existing conditional alignment loss, which can be formalized in Equation (\ref{equ:align1}).
\begin{equation}
    \mathcal{L}_{a} = \frac{1}{C} \sum_{i=1}^C ||\hat{\rvz}_{3, \mathcal{S}}^{(i)} - \hat{\rvz}_{3,\mathcal{T}}^{(i)}||_2,
    \label{equ:align1}
\end{equation}
However, conventional conditional alignment usually suffers from two drawbacks including misestimated centroid and low-quality pseudo-labels. First, the conditional alignment method implicitly assumes that the feature centroids from different domains are the same. But it is hard to estimate the correct centroids of the target domain for the class with low probability density. Second, conditional alignment heavily relies on the quality of the pseudo label. But existing methods usually use pseudo labels without any discrimination, which might result in false alignment. To solve these problems, we consider two types of reweighting.

\textbf{Distribution-based Reweighting for Misestimated Centroid:} Although the conditional alignment method implicitly assumes that the feature centroids from different domains are the same, it is hard to estimate the correct centroids of the target domain for the class with low probability density. To address this challenge, one straightforward solution is to consider the label distribution of the target domain. To achieve this, we employ the technique of black box shift estimation method (BBSE) \cite{lipton2018detecting} to estimate the label distribution from the target domain $p_{\hat{\rvy}}$.
So we use the estimated label distribution to reweight the conditional alignment loss in Equation (\ref{equ:align1}).

\textbf{Entropy-based Reweighting for Low-quality Pseudo-labels:} conditional alignment heavily relies on the quality of the pseudo label. However, existing methods usually use pseudo labels without any discrimination, which might result in false alignment. To address this challenge, we consider the prediction uncertainty of each class in the target domain. Technologically, for each sample in the target dataset, we calculate the entropy-based weights via the prediction results which are shown as $w^{(i)}$ in Equation (\ref{equ:align3}).

By combining the distribution-based weights and the entropy-based weight, we can obtain the class-aware conditional alignment as shown in Equation (\ref{equ:align3}). Hence the total loss of the Subspace Identification Guarantee model can be formalized as follows:
\begin{equation}
    \mathcal{L}_{total} = \mathcal{L}_y + \beta \mathcal{L}_{vae} + \alpha \mathcal{L}_a ,
\end{equation}
where $\alpha,\beta$ denote the hyper-parameters.

\section{Experiments}
\subsection{Experiments on Simulation Data}
In this subsection, we illustrate the experiment results of simulation data to evaluate the theoretical results of subspace identification in practice. 

\subsubsection{Experimental Setup}

\textbf{Data Generation.} We generate the simulation data for binary classification with 8 domains. To better evaluate our theoretical results, we follow the data generation process in Figure \ref{fig:data2_paper}, which includes two types of latent variables, i.e., domain-specific latent variables $\rvz_s$ and domain-invariant latent variables $\rvz_c$. We let the dimensions of $\bm{z}_s$ and $\bm{z}_c$ be both 2. Moreover, $\bm{z}_s$ are sampled from $u$ different mixture of Gaussians, and $\bm{z}_c$ are sampled from a factorized Gaussian distribution. We let the data generation process from latent variables to observed variables be MLPs with the Tanh activation function. We further split the simulation dataset into the training set, validation set, and test set.

\textbf{Evaluation Metrics.} First, we employ the accuracy of the target domain data to measure the classification performance of the model. Second, we compute the Mean Correlation Coefficient (MCC) between the ground-truth $\bm{z}_s$ and the estimated $\hat{\bm{z}}_s$ on the test dataset to evaluate the component-wise identifiability of domain-specific latent variables. A higher MCC denotes the better identification performance the model can achieve. Third, to evaluate the performance of subspace identifiability of domain-specific latent variables, we first use the estimated $\hat{\bm{z}}_s$ from the validation set to regress each dimension of the ground-truth $\bm{z}_s$ from the validation set with the help of a MLPs. Sequentially, we take the $\hat{\bm{z}}_s$ from the test set as input to estimate how well the MLPs can reconstruct the ground-truth $\bm{z}_s$ from the test set, so we employ Root Mean Square Error (RMSE) to measure the extent of subspace identification. A low RMSE denotes that there exists a transformation $h_i$ between $\bm{z}_{s,{i}}$ and $\hat{\bm{z}}_{s,1}, \hat{\bm{z}}_{s,2}$, i.e. $\bm{z}_{s,{i}}=h_i(\hat{\bm{z}}_{s,1}, \hat{\bm{z}}_{s,2}), i\in\{1,2\}$. For the scenario where the number of domains is less than 8, we first fix the target domain and then try all the combinations of the source domains. And we publish the average performance of all the combinations. We repeat each experiment over 3 random seeds.

\subsubsection{Results and Discussion}

\begin{wraptable}{r}{11.5cm}
    \centering
     \vspace{-0.3cm}
    \caption{Experiments results on simulation data. }
    \vspace{0.3cm}
    \label{tab:simulation}
    
\begin{tabular}{l|c|ccc}
\toprule
\multicolumn{1}{c|}{State}                                                               & U & ACC            & MCC            & RMSE           \\ \midrule 
\multirow{3}{*}{\begin{tabular}[c]{@{}l@{}}Component-wise\\ Identification\end{tabular}} & 8 & 0.9982(0.0004) & 0.9037(0.0087) & 0.0433(0.0051) \\ & 6 & 0.9982(0.0007) & 0.8976(0.0162) & 0.0439(0.0073) \\& 5 & 0.9982(0.0007) & 0.8973(0.0131) & 0.0441(0.0055) \\ \midrule
\multirow{2}{*}{\begin{tabular}[c]{@{}l@{}}Subspace \\ Identification\end{tabular}}      & 4 & 0.9233(0.2039) & 0.8484(0.1452) & 0.0582(0.0431) \\& 3 & 0.8679(0.2610) & 0.8077(0.1709) & 0.0669(0.0482) \\ \midrule
No Identification& 2 & 0.5978(0.3039) & 0.6184(0.2093) & 0.1272(0.0608) \\ \bottomrule
\end{tabular}
\end{wraptable}

The experimental results of the simulation dataset are shown in Table \ref{tab:simulation}. According to the experiment results, we can obtain the following conclusions: 1) We can find that the values of MCC increase along with the number of domains. Moreover, the values of MCC are high (around 0.9) and stable when the number of domains is larger than 5. This result corresponds to the theoretical result of component-wise identification, where a certain number of domains (i.e. $2n+1$) are necessary for component-wise identification. 2) We can find that the values of RMSE decrease along with the number of domains. Furthermore, the values of RMSE are low and stable (less than 0.07) when the number of domains is larger than 3, but it drops sharply when $u=2$. These experimental results coincide with the theoretical results of subspace identification as well as the intuition where a certain number of domains are necessary for subspace identification (i.e. $n_s+1$). 3) According to the experimental results of ACC, we can find that the accuracy grows along with the number of domains and its changing pattern is relevant to that of RMSE, i.e., the performance is stable when the number of domains is larger than 3. \textcolor{black}{The ACC results also indirectly support the results of subspace identification, since one straightforward understanding of subspace identification is that the domain-specific information is preserved in $\hat{z}_s$. 
And the latent variables are well disentangled with the help of subspace identification, which benefits the model performance.}

\subsection{Experiments on Real-world Data}
\subsubsection{Experimental Setup}

\begin{table}[]
\setlength{\abovecaptionskip}{5pt}
\caption{Classification results on the Office-Home and ImageCLEF datasets. For the Office-Home dataset, We employ ResNet50 as the backbone network. For the ImageCLEF dataset, we employ ResNet18 as the backbone network. }
\label{tab:exp_office_clef}
\setlength{\tabcolsep}{1.55mm}{
\begin{tabular}{l|ccccc|cccc}
\toprule
\multicolumn{1}{c|}{\multirow{2}{*}{Model}} & \multicolumn{5}{c|}{Office-Home}               & \multicolumn{4}{c}{ImageCLEF}                                 \\ \cmidrule{2-10}
                       & Art  & Clipart & Product & RealWorld & Average & P                    & C    & I    & Average                  \\ \midrule
\textbf{Source Only \cite{he2016deep}}           & 64.5 & 52.3    & 77.6    & 80.7      & 68.8    & 77.2                 & 92.3 & 88.1 & 85.8                     \\
\textbf{DANN \cite{ganin2015unsupervised}}                   & 64.2 & 58.0      & 76.4    & 78.8      & 69.3    & 77.9                 & 93.7 & 91.8 & 87.8                     \\
\textbf{DAN \cite{long2015learning}}                   & 68.2 & 57.9    & 78.4    & 81.9      & 71.6    & 77.6                 & 93.3 & 92.2 & 87.7                     \\
\textbf{DCTN \cite{xu2018deep}}                  & 66.9 & 61.8    & 79.2    & 77.7      & 71.4    & 75.0                   & 95.7 & 90.3 & 87.0                       \\
\textbf{MFSAN \cite{zhu2019aligning}}                & 72.1 & 62.0      & 80.3    & 81.8      & 74.1    & 79.1                 & 95.4 & 93.6 & 89.4                     \\
\textbf{WADN \cite{shui2021aggregating}}                 & 75.2 & 61.0      & 83.5    & 84.4      & 76.1    & 77.7 & 95.8 & 93.2 & 88.9    \\
\textbf{iMSDA \cite{kong2022partial}}                 & 75.4 & 61.4    & 83.5    & 84.4      & 76.2    & 79.2                 & 96.3 & \textbf{94.3} & 90.0                       \\ \midrule
\textbf{SIG}                    & \textbf{76.4} & \textbf{63.9}    & \textbf{85.4}    & \textbf{85.8}      & \textbf{77.8}    & \textbf{79.3}                 & \textbf{97.3} & \textbf{94.3} & \textbf{90.3} \\ \bottomrule
\end{tabular}}
\end{table}

\textbf{Datasets:} We consider four benchmarks: Office-Home, PACS, ImageCLEF, and DomainNet. For each dataset, we let each domain be a target domain and the other domains be the source domains. For the DomainNet dataset, we equip a cross-attention module to the ResNet101 backbone networks for better usage of domain knowledge. We also employ the alignment of MDD \cite{zhang2019bridging}. For the Office-Home and ImageCLEF datasets, we employ the pre-trained ResNet50 with an MLP-based classifier. For the PACS dataset, we use ResNet18 with an MLP-based classifier. The implementation details are provided in the Appendix \ref{app:imple}. We report the average results over 3 random seeds.

\begin{wraptable}{r}{8.7cm}
\vspace{-0.3cm}
\caption{Classification results on the PACS datasets. We employ ResNet18 as the backbone network. Experiment results of other compared methods are taken from 
(\cite{kong2022partial}). }
\vspace{0.3cm}
    \label{tab:pacs_main}
    \begin{tabular}{l|cccc|c}
    \toprule
        \textbf{Model} & A & C & P & S & Average \\ \midrule
        \textbf{Source Only \cite{he2016deep}} & 74.9 & 72.1 & 94.5 & 64.7  & 76.7 \\ 
        \textbf{DANN \cite{ganin2015unsupervised}} & 81.9  & 77.5 & 91.8 & 74.6 & 81.5 \\ 
        \textbf{MDAN \cite{zhao2018adversarial}} & 79.1  & 76.0  & 91.4 & 72.0  & 79.6 \\ 
        \textbf{WBN \cite{mancini2018boosting}} & 89.9  & 89.7  & 97.4  & 58.0  & 83.8 \\ 
        \textbf{MCD \cite{saito2018maximum}} & 88.7  & 88.9 & 96.4  & 73.9  & 87.0 \\ 
        \textbf{M3SDA \cite{peng2019moment}} & 89.3  & 89.9 & 97.3  & 76.7  & 88.3 \\ 
        \textbf{CMSS \cite{yang2020curriculum}} & 88.6  & 90.4  & 96.9  & 82.0 & 89.5 \\ 
        \textbf{LtC-MSDA \cite{wang2020learning}} & 90.1 & 90.4 & 97.2 & 81.5 & 89.8 \\ 
        \textbf{T-SVDNet \cite{li2021t}} & 90.4 & 90.6 & 98.5 & 85.4 & 91.2 \\ 
        \textbf{iMSDA \cite{kong2022partial}} & 93.7  & 92.4  & 98.4 & 89.2 & 93.4 \\ \midrule
        \textbf{SIG} & \textbf{94.1}  & \textbf{93.6}  & \textbf{98.6} & \textbf{89.5}  & \textbf{93.9} \\ \bottomrule
    \end{tabular}
\end{wraptable}
\textbf{Baselines:} Besides the classical approaches for single source domain adaptation like DANN \cite{ganin2015unsupervised}, DAN \cite{long2015learning}, MCD \cite{saito2018maximum}, and ADDA \cite{tzeng2017adversarial}. We also compare our method with several state-of-the-art multi-source domain adaptation methods, for example, MIAN-$\gamma$ \cite{park2021information}, T-SVDNet \cite{li2021t}, LtC-MSDA \cite{wang2020learning}, SPS \cite{wang2022self}, and PFDA \cite{fu2021partial}. Moreover, we further consider the WADN \cite{shui2021aggregating}, which is devised for the target shift of multi-source domain adaptation. For a fair comparison, we employ the same pre-train backbone networks instead of the pre-trained features for WADN in the original paper. We also consider the latest iMSDA \cite{kong2022partial}, which addresses the MSDA via component-wise identification. 

\subsubsection{Results and Discussion}



Experimental results on Office-Home, ImageCLEF, PACS, and DomainNet are shown in Table \ref{tab:exp_office_clef}, \ref{tab:pacs_main}, and \ref{tab:domainnet_main}, respectively. Experiment results of other compared methods are provided in Appendix \ref{app:real_exp}. 

According to the experiment results of the Office-Home dataset on the left side of Table \ref{tab:exp_office_clef}, our SIG model significantly outperforms all other baselines on all the transfer tasks. Specifically, our method outperforms the most competitive baseline by a clear margin of $1.3\%-4\%$ and promotes the classification accuracy substantially on the hard transfer task, e.g. Clipart. It is noted that our method achieves a better performance than that of WADN, which is designed for the target shift scenario. This is because our method not only considers how the domain variables influence the distribution of labels but also identifies the latent variables of the data generation process. Moreover, our SIG method also outperforms iMSDA, indirectly reflecting that the proposed data generation process is closer to the real-world setting and the subspace identification can achieve better disentanglement performance under limited auxiliary variables.

For datasets like ImageCLEF and PACS, our method also achieves the best-averaged results. In detail, we achieved a comparable performance in all the transfer tasks in the ImageCLFE dataset. In the PACS dataset, our SIG method still performs better than the latest compared methods like iMSDA and T-SVDNet in some challenging transfer tasks like Cartoon.
Finally, we also consider the most challenging dataset, DomainNet, which contains more classes and more complex domain shifts. Results in Table \ref{tab:domainnet_main} show the significant performance of the proposed SIG method, which provides $3.3\%$ averaged promotion, although the performance in the task of Sketch is slightly lower than that of DRT+ST. Compared with iMSDA, our SIG overpasses by a large margin under a more general data generation process.

\textbf{Ablation Study:}
To evaluate the effectiveness of individual loss terms, we also devise the two model variants. 1) \textbf{SIG-sem}: remove the class-aware alignment loss. 2) \textbf{SIG-vae}: remove the reconstruction loss and the KL divergence loss.
Experiment results on the Office-Home dataset are shown in Figure \ref{fig:ablation}. We can find that the class-aware alignment loss plays an important role in the model performance, reflecting that the class-aware alignment can mitigate the influence of target shift. We also discover that incorporating the reconstruction and KL divergence has a positive impact on the overall performance of the model, which shows the necessity of modeling the marginal distributions.


\begin{table}[t]
\setlength{\abovecaptionskip}{5pt}
    \centering
    \caption{Classification results on the DomainNet datasets. We employ ResNet101 as the backbone network. Experiment results of other compared methods are taken from (\cite{li2021dynamic} and \cite{wang2022self}).}\label{tab:domainnet_main}
    \setlength{\tabcolsep}{1.9mm}{
    \begin{tabular}{l|cccccc|c}
    \toprule
        \textbf{Model} & Clipart & Infograph & Painting & Quickdraw & Real & Sketch & Average \\ \midrule
        \textbf{Source Only \cite{he2016deep}} & 52.1 & 23.4  & 47.6  & 13.0  & 60.7  & 46.5 & 40.6 \\ 
        \textbf{ADDA \cite{tzeng2017adversarial}} & 47.5  & 11.4  & 36.7  & 14.7  & 49.1  & 33.5  & 32.2 \\ 
        \textbf{MCD \cite{saito2018maximum}} & 54.3  & 22.1  & 45.7  & 7.6  & 58.4  & 43.5  & 38.5 \\ 
        \textbf{DANN \cite{ganin2015unsupervised}} & 60.6  & 25.8  & 50.4  & 7.7 & 62.0 & 51.7  & 43 \\ 
        \textbf{DCTN \cite{xu2018deep}} & 48.6  & 23.5  & 48.8  & 7.2 & 53.5  & 47.3  & 38.2 \\ 
        \textbf{M$^3$SDA-$\beta$ \cite{peng2019moment}} & 58.6  & 26.0  & 52.3  & 6.3  & 62.7  & 49.5  & 42.6 \\ 
        \textbf{ML\_MSDA \cite{li2020mutual}} & 61.4  & 26.2  & 51.9  & 19.1  & 57.0 & 50.3  & 44.3 \\ 
        \textbf{meta-MCD \cite{li2020online}} & 62.8  & 21.4  & 50.5  & 15.5 & 64.6  & 50.4  & 44.2 \\ 
        \textbf{LtC-MSDA \cite{wang2020learning}} & 63.1  & 28.7  & 56.1  & 16.3  & 66.1 & 53.8  & 47.4 \\ 
        \textbf{CMSS \cite{yang2020curriculum}} & 64.2  & 28.0  & 53.6  & 16.9  & 63.4  & 53.8  & 46.5 \\ 
        \textbf{DRT+ST \cite{li2021dynamic}} & 71.0  & 31.6  & 61.0  & 12.3  & 71.4  & \textbf{60.7}  & 51.3 \\ 
        \textbf{SPS \cite{wang2022self}} & 70.8 & 24.6 & 55.2 & 19.4 & 67.5 & 57.6 & 49.2 \\ 
        \textbf{PFDA \cite{fu2021partial}} & 64.5 & 29.2 & 57.6 & 17.2 & 67.2 & 55.1 & 48.5 \\ 
        \textbf{iMSDA \cite{kong2022partial}} & 68.1 & 25.9 & 57.4 & 17.3 & 64.2 & 52.0 & 47.5 \\ 
        \midrule
        \textbf{SIG} & \textbf{72.7}  & \textbf{32.0}  & \textbf{61.5}  & \textbf{20.5} & \textbf{72.4}    & 59.5  & 53.0 \\ \bottomrule
    \end{tabular}
    }
\end{table}

\begin{figure}[t]
  \centering
	\includegraphics[width=\columnwidth]{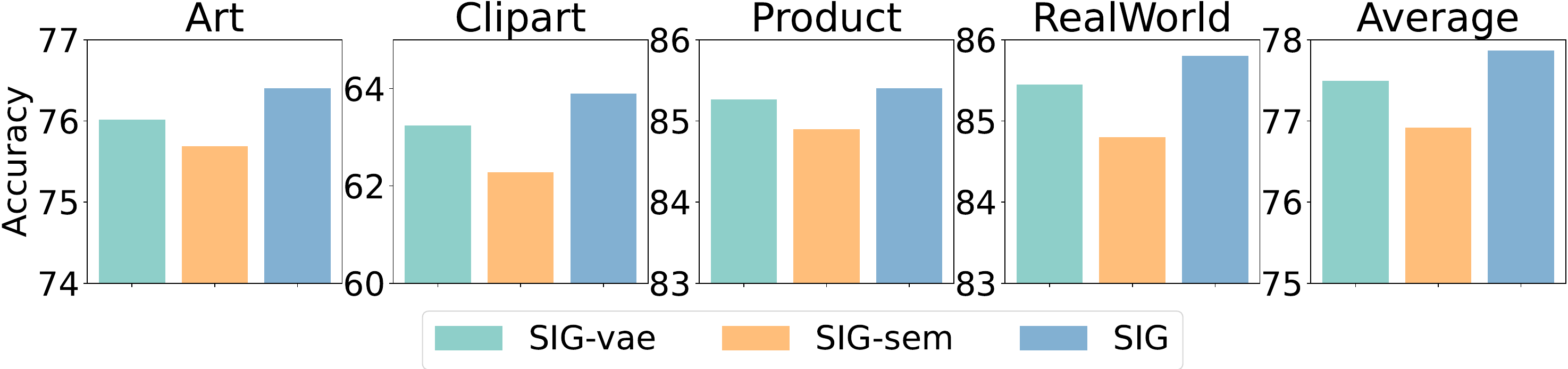}
  \caption{Ablation study on the Office-Home dataset. we explore the impact of different loss terms.} 
  \label{fig:ablation} 
\end{figure}


\vspace{-7pt}
\section{Conclusion}
\vspace{-10pt}
This paper presents a general data generation process for multi-source domain adaptation, which coincides with real-world scenarios. Based on this data generation process, we prove that the changing latent variables are subspace identifiable, which provides a novel solution for disentangled representation. Compared with the existing methods, the proposed subspace identification theory requires fewer auxiliary variables and frees the model from the monotonic transformation of latent variables, making it possible to apply the proposed method to real-world data. Experiment results on several mainstream benchmark datasets further evaluate the effectiveness of the proposed subspace identification guaranteed model. In summary, this paper takes a meaningful step for causal representation learning.
\textbf{Broader Impacts:} SIG disentangles the latent variables to create a model that is more transparent, thereby aiding in the reduction of bias and the promotion of fairness. 
\textbf{Limitation:}
However, the proposed subspace identification still requires several assumptions that might not be met in real-world scenarios. Therefore, how further to relax the assumptions, i.e., conditional independent assumption, would be an interesting future direction.

\section{Acknowledgements} 
We are very grateful to the anonymous reviewers for their help in improving the paper. This research was supported in part by the National Key R\&D Program of China (2021ZD0111501), the National Science Fund for Excellent Young Scholars (62122022), Natural Science Foundation of China (61876043, 61976052), the major key project of PCL (PCL2021A12). This project is also partially supported by NSF Grant 2229881, the National Institutes of Health (NIH) under Contract R01HL159805, a grant from Apple Inc., a grant from KDDI Research Inc., and generous gifts from Salesforce Inc., Microsoft Research, and Amazon Research.


\bibliography{main}
\bibliographystyle{plain}
\clearpage

\renewcommand{\arraystretch}{1.5}

\maketitle

\appendix
{\footnotesize 
\tableofcontents
}

\section{Identify Target Joint Distribution}
\label{app:iden_joint_dist}
We show how to derive the conditions of identifying the target joint distribution with the help of the proposed data generation process, which is shown in Equation (\ref{equ:tgt_joint_dist}).

\begin{equation}
\begin{split}    p_{\rvx,\rvy|\rvu_{\mathcal{T}}}&\overset{(1)}{=}\int_{\bm{\rvz}_{1}}\int_{\rvz_{2}}\int_{\rvz_{3}}\int_{\rvz_{4}}p_{\rvx,\rvy,\bm{\rvz}_{1},\bm{\rvz}_{2},\bm{\rvz}_{3},\bm{\rvz}_{4}|\rvu_{\mathcal{T}}}d\bm{\rvz}_{1}d\bm{\rvz}_{2}d\bm{\rvz}_{3}d\bm{\rvz}_{4}\\&
\overset{(2)}{=}\int_{\bm{\rvz}_{1}}\int_{\rvz_{2}}\int_{\rvz_{3}}\int_{\rvz_{4}}p_{\rvx,\bm{\rvz}_{1},\bm{\rvz}_{2},\bm{\rvz}_{3},\bm{\rvz}_{4}|\rvy,\rvu_{\mathcal{T}}}\cdot p_{\rvy|\rvu_{\mathcal{T}}}d\bm{\rvz}_{1}d\bm{\rvz}_{2}d\bm{\rvz}_{3}d\bm{\rvz}_{4}
\\&\overset{(3)}{=}\int_{\bm{\rvz}_{1}}\int_{\rvz_{2}}\int_{\rvz_{3}}\int_{\rvz_{4}}p_{\rvx|\bm{\rvz}_{1},\bm{\rvz}_{2},\bm{\rvz}_{3},\bm{\rvz}_{4}}\cdot p_{\bm{\rvz}_{1},\bm{\rvz}_{2},\bm{\rvz}_{3},\bm{\rvz}_{4}|\rvy,\rvu_{\mathcal{T }}} \cdot p_{\rvy|\rvu_{\mathcal{T}}} d\bm{\rvz}_{1}d\bm{\rvz}_{2}d\bm{\rvz}_{3}d\bm{\rvz}_{4}.
\end{split}
\label{equ:tgt_joint_dist}
\end{equation}

The derivation in Equation (\ref{equ:tgt_joint_dist}) can be separated into three steps. $(1)$ We introduce the latent variables $\rvz_1,\rvz_2,\rvz_3$, and $\rvz_4$, which have mentioned in Section 2.1. $(2)$ We factorize the joint distribution in $(1)$ into $p_{\rvx,\rvz_1,\rvz_2,,\rvz_3,,\rvz_4|\rvy,\rvu_{\mathcal{T}}}$ and $p_{\rvy|\rvu_{\mathcal{T}}}$ with the help of Bayes Rule. $(3)$, we further use Bayes Rule to factorize $p_{\rvx,\rvz_1,\rvz_2,,\rvz_3,,\rvz_4|\rvy,\rvu_{\mathcal{T}}}$. Since $\rvx$ is independent of $\rvu, \rvy$ given $\rvz_1,\rvz_2,,\rvz_3,,\rvz_4$, we can obtain $p_{\rvx|\rvz_1,\rvz_2,,\rvz_3,,\rvz_4}$.

The aforementioned factorization tells us that we need to model three distributions to identify the target joint distribution.
First, we need to model $p_{\rvx|\bm{\rvz}_{1},\bm{\rvz}_{2},\bm{\rvz}_{3},\bm{\rvz}_{4}}$, implying that we need to model the conditional distribution of observed data give latent variables, which coincides with a generative model for observed data. Second, we need to estimate the label pseudo distribution of target domain $p_{\rvy|\rvu_{\mathcal{T}}}$. Third, we need to model $p_{\bm{\rvz}_{1},\bm{\rvz}_{2},\bm{\rvz}_{3},\bm{\rvz}_{4}|\rvy,\rvu_{\mathcal{T}}}$ meaning that the latent variables should be identified with theoretical guarantees. In the next section, we will introduce how to identify these latent variables with subspace identification block-wise identification results. 



\section{Proof of the Identification of latent variables}
\subsection{Proof of Subspace Identification}
\label{app:proof_subspace}

\begin{figure}[!ht]
	\centering
\includegraphics[width=0.18\columnwidth]{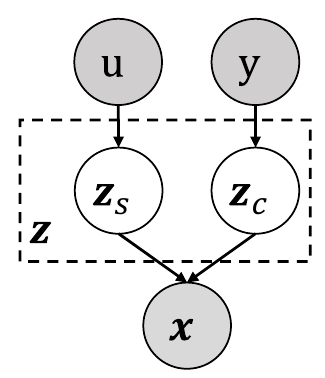}
	\caption{A simple data generalization process for introducing subspace identification.}
	\label{fig:data2}
\end{figure}

In this subsection, we provide proof of the subspace identification based on the data generation process in Figure \ref{fig:data2}.

\begin{theorem}
\label{the1}
(\textbf{Subspace Identification of $\rvz_s$.}) We follow the data generation process in Figure \ref{fig:data2} and make the following assumptions:
\begin{itemize}[leftmargin=*,itemsep=-3pt]
    \item A1 (\underline{Smooth and Positive Density}): The probability density function of latent variables is smooth and positive, i.e., $p_{\rvz|\rvu}>0$ over $\mathcal{Z}$ and $\mathcal{U}$.
    \item A2 (\underline{Conditional independent}): Conditioned on $\rvu$, each $z_i$ is independent of any other $z_j$ for $i,j \in \{1,\cdots,n\}, i\neq j$, i.e. $\log p_{\rvz|\rvu}(\rvz|\rvu)=\sum_i^n q_i(z_{i},\rvu)$ where $q_i(z_{i},\rvu)$ is the log density of the conditional distribution, i.e., $q_i:\log p_{z_{i}|\rvu}$.
    \item A3 (\underline{Linear independence}): For any $\rvz_s\in \mathcal{Z}_s\subseteq \mathbb{R}^{n_s}$, there exist $n_s+1$ values of $\rvu$, i.e., $\rvu_j$ with $j=0,1,\cdots,n_s$, such that these $n_s$ vectors $\bm{\rvw}(\rvz,\rvu_j)-\bm{\rvw}(\rvz,\rvu_0)$ with $j=1,\cdots,n_s$ are linearly independent, where vector $\bm{\rvw}(\rvz,\rvu_j)$ is defined as follows:
    \begin{equation}
        \bm{\rvw}(\rvz,\rvu)=\left(\frac{\partial q_1(z_{1},\rvu)}{\partial z_{1}},\cdots, \frac{\partial q_i(z_{i},\rvu)}{\partial z_{i}},\cdots \frac{\partial q_{n_s}(z_{n_s},\rvu)}{\partial z_{n_s}}\right),
    \end{equation}
\end{itemize}

By modeling the aforementioned data generation process, $\rvz_s$ is subspace identifiable.
\end{theorem}

\begin{proof}
We begin with the matched marginal distribution $p_{\rvx|\rvu}$ to bridge the relation between $\rvz$ and $\hat{\rvz}$. Suppose that $\hat{g}:\mathcal{Z}\rightarrow \mathcal{X}$ is a invertible estimated generating function, we have Equation (\ref{equ:t1_1}).
\begin{equation}
\label{equ:t1_1}
\forall \rvu\in \mathcal{U}, \quad p_{\hat{\rvx}|\rvu}=p_{\rvx|\rvu} \Longleftrightarrow p_{\hat{g}(\hat{\rvz})|\rvu}=p_{g(\rvz)|\rvu}.
\end{equation}
Sequentially, by using the change of variables formula, we can further obtain Equation (\ref{equ:t1_2})
\begin{equation}
\label{equ:t1_2}
    p_{\hat{g}(\hat{\rvz}|\rvu)}=p_{g(\rvz|\rvu)} \Longleftrightarrow  p_{g^{-1}\circ g(\hat{\rvz})|\rvu}|\mJ_{g^{-1}}| = p_{\rvz|\rvu}|\mJ_{g^{-1}}| \Longleftrightarrow p_{h(\hat{\rvz})|\rvu}=p_{\rvz|\rvu},
\end{equation}
where $h:=g^{-1}\circ g$ is the transformation between the ground-true and the estimated latent variables, respectively. $\mJ_{g^{-1}}$ denotes the absolute value of Jacobian matrix determinant of $g^{-1}$. Since we assume that $g$ and $\hat{g}$ are invertible, $|\mJ_{g^{-1}}|\neq 0$ and $h$ is also invertible.

According to A2 (conditional independent assumption), we can have Equation (\ref{equ:t1_3}).
\begin{equation}
\label{equ:t1_3}
    p_{\rvz|\rvu}(\rvz|\rvu)=\prod_{i=1}^{n} p_{z_i|\rvu}(z_i|\rvu); \quad p_{\hat{\rvz}|\rvu}(\hat{\rvz}|\rvu) = \prod_{i=1}^n p_{\hat{z}_i|\rvu}(\hat{z}_i|\rvu).
\end{equation}
For convenience, we take logarithm on both sides of Equation (\ref{equ:t1_3}) and further let $q_i:=\log p_{z_i|\rvu}, \hat{q}_i:=\log p_{\hat{z}_i|\rvu}$. Hence we have:
\begin{equation}
\label{equ:t1_4}
\log p_{\rvz|\rvu}(\rvz|\rvu)=\sum_{i=1}^n q_i(z_i,\rvu);\quad \log p_{\hat{\rvz}|\rvu}=\sum_{i=1}^n \hat{q}_i(\hat{z}_i,\rvu).
\end{equation}

By combining Equation (\ref{equ:t1_4}) and Equation (\ref{equ:t1_2}), we have:
\begin{equation}
\label{equ:t1_5}
    p_{\rvz|\rvu}=p_{h(\hat{\rvz}|\rvu)}\Longleftrightarrow p_{\hat{\rvz}|\rvu}=p_{\rvz|\rvu}|\mJ_{h^{-1}}|\Longleftrightarrow \sum_{i=1}^n q_i({z_i,\rvu}) + \log |\mJ_{h^{-1}}| = \sum_{i=1}^n \hat{q}_i(\hat{z}_i,\rvu),
\end{equation}
where $\mJ_{h^{-1}}$ are the Jacobian matrix of $h^{-1}$.


Sequentially, we take the first-order derivative with $\hat{z}_j$ on Equation (\ref{equ:t1_5}), where $j\in\{n_s+1,\cdots,n\}$, and have
\begin{equation}
\label{equ:app_1_order_der}
    \sum_{i=1}^n\frac{\partial q_i(z_i,\rvu)}{\partial z_i}\cdot \frac{\partial z_i}{\partial \hat{z}_j} + \frac{\partial\log |\mJ_{h^{-1}}|}{\partial \hat{z}_j} = \frac{\partial q_j(\hat{z}_j, \rvu)}{\partial \hat{z}_j}.
\end{equation}

Suppose $\rvu=u_0,u_1,\cdots,u_{n_s}$, we subtract the Equation (\ref{equ:app_1_order_der}) corresponding to $u_k$ with that corresponds to $u_0$, and we have:
\begin{equation}
    \sum_{i=1}^{n} \left(\frac{\partial q_i(z_i,u_k)}{\partial z_i}-\frac{\partial q_i(z_i,u_0)}{\partial z_i}\right)\cdot \frac{\partial z_i}{\partial \hat{z}_j} = \frac{\partial \hat{q}_j(\hat{z}_j,u_k)}{\partial \hat{z}_j} - \frac{\partial \hat{q}_j(\hat{z}_j,u_0)}{\partial \hat{z}_j}.
\end{equation}
Since the distribution of estimated $\hat{\rvz}_j$ does not change across different domains, $\frac{\partial \hat{q}_j(\hat{z}_j,u_k)}{\partial \hat{z}_j} - \frac{\partial \hat{q}_j(\hat{z}_j,u_0)}{\partial \hat{z}_j}=0$. Since $\frac{\partial q_i(z_i,u_k)}{\partial z_i}$ does not change across different domains, $\frac{\partial q_i(z_i,u_k)}{\partial z_i}=\frac{\partial q_i(z_i,u_0)}{\partial z_i}$ for $i\in \{n_s+1,\cdots,n\}$. So we have
\begin{equation}
\sum_{i=1}^{n_s} \left(\frac{\partial q_i(z_i,u_k)}{\partial z_i}-\frac{\partial q_i(z_i,u_0)}{\partial z_i}\right)\cdot \frac{\partial z_i}{\partial \hat{z}_j} = 0.
\end{equation}

Based on the linear independence assumption (A3), the linear system is a $n_s \times n_s$ full-rank system. Therefore, the only solution is $\frac{\partial z_i}{\partial \hat{z}_j}=0$ for $i\in\{1,\cdots,n_s\}$ and $j\in\{n_s+1,\cdots,n\}$.

Since $h(\cdot)$ is smooth over $\mathcal{Z}$, its Jacobian can be formalized as follows
\begin{equation}
\begin{gathered}
    \mJ_{h}=\begin{bmatrix}
    \begin{array}{c|c}
        \textbf{A}:=\frac{\partial \rvz_s}{\partial \hat{\rvz}_s} & \textbf{B}:=\frac{\partial \rvz_s}{\partial \hat{\rvz}_c} \\ \hline
        \textbf{C}:=\frac{\partial \rvz_c}{\partial \hat{\rvz}_s} & \textbf{D}:=\frac{\partial \rvz_c}{\partial \hat{\rvz}_c}.
    \end{array}
    \end{bmatrix}
\end{gathered}
\end{equation}
Note that $\frac{\partial z_i}{\partial \hat{z}_j}=0$ for $i\in\{1,\cdots,n_s\}$ and $j\in\{n_s+1,\cdots,n\}$ means that $\textbf{B}=0$. Since $h(\cdot)$ is invertible, $\mJ_{h}$ is a full-rank matrix. Therefore, for each $z_{s,i}, i\in\{1,\cdots,n_s\}$, there exists a $h_i$ such that  $z_{s,i}=h_i(\hat{\rvz})$.
\end{proof}

\subsection{Proof of Corollary1.1}
\label{corollary}
\begin{corollary}
We follow the data generation in Section 3.1, and make the following assumptions which are similar to A1-A3:
\item A4 (\underline{Smooth and Positive Density}): The probability density function of latent variables is smooth and positive, i.e., $p_{\rvz|\rvu, \rvy}>0$ over $\mathcal{Z}$, $\mathcal{U}$, and $\mathcal{Y}$.
\item A5 (\underline{Conditional independent}): Conditioned on $\rvu$ and $\rvy$, each $z_i$ is independent of any other $z_j$ for $i,j \in \{1,\cdots,n\}, i\neq j$, i.e. $\log p_{\rvz|\rvu,\rvy}(\rvz|\rvu,\rvy)=\sum_i^n q_i(z_{i},\rvu,\rvy)$ where $q_i(z_{i},\rvu,\rvy)$ is the log density of the conditional distribution, i.e., $q_i:\log p_{z_{i}|\rvu,\rvy }$.
\item A6 (\underline{Linear independence}): For any $\rvz\in\mathcal{Z}\subseteq\mathbb{R}^n$, there exists $n_1+n_2+n_3+1$ combination of $(\rvu,\rvy)$, i.e. $j=1,\cdots, U$ and $ c=1, \cdots, C$ and $ U\times C + 1=n_1+n_2+n_3$, where $U$ and $C$ denote the number of source domains and the number of labels. such that these $n'=n_1+n_2+n_3$ vectors $\bm{\rvw}(\rvz,\rvu_j,\rvy_c)-\bm{\rvw}(\rvz,\rvu_0,\rvy_0)$ are linearly independent, where $\bm{\rvw}(\rvz,\rvu_j,\rvy_c)$ is defined as follows:
\begin{equation}
    \bm{\rvw}(\rvz,\rvu_j,\rvy_c)=\left(\frac{\partial q_1(z_{1},\rvu,\rvy)}{\partial z_{1}},\cdots, \frac{\partial q_i(z_{i},\rvu,\rvy)}{\partial z_{i}},\cdots \frac{\partial q_{n'}(z_{n'},\rvu,\rvy)}{\partial z_{n'}}\right).
\end{equation}
By modeling the aforementioned data generation process, $\rvz_2$ is subspace identifiable, and $\rvz_1, \rvz_3$ can be reconstructed from $\hat{\rvz}_1, \hat{\rvz}_2$ and $\hat{\rvz}_2, \hat{\rvz}_3$, respectively.
\end{corollary}
\begin{proof}
We begin with the match marginal distribution $p_{\rvx|\rvu,\rvy}$ to bridge the relation between $\rvz$ and $\hat{\rvz}$. Suppose that $\hat{g}:\mathcal{Z}\rightarrow\mathcal{X}$ is an invertible estimated generating function, we have Equation (\ref{equ:t2_1}).
\begin{equation}
\label{equ:t2_1}
   \forall \rvu\in \mathcal{U}, \rvy \in \mathcal{Y}, p_{\hat{\rvx}|\rvu,\rvy}=p_{\rvx|\rvu,\rvy} \Longleftrightarrow p_{\hat{g}(\hat{\rvz})|\rvu,\rvy}=p_{g(\rvz)|\rvu,\rvy}.
\end{equation}
Sequentially, by using the change of variables formula, we can further obtain Equation(\ref{equ:t2_2}).
\begin{equation}
\label{equ:t2_2}
    p_{\hat{g}(\hat{\rvz})|\rvu,\rvy}=p_{g(\rvz)|\rvu,\rvy} \Longleftrightarrow  p_{g^{-1}\circ g(\hat{\rvz})|\rvu,\rvy}|\mJ_{g^{-1}}| = p_{\rvz|\rvu,\rvy}|\mJ_{g^{-1}}| \Longleftrightarrow p_{h(\hat{\rvz})|\rvu,\rvy}=p_{\rvz|\rvu,\rvy},
\end{equation}
where $h:=g^{-1}\circ g$ is the transformation between the ground-true and the estimated latent variables. $\mJ_{g^{-1}}$ denotes the absolute value of Jacobian matrix determinant of $g^{-1}$. Since we assume that $g$ and $\hat{g}$ are invertible, $|\mJ_{g^{-1}}|\neq 0$ and $h$ is also invertible.

According to A5 (conditional independent assumption), we can have Equation (\ref{equ:t2_3}).
\begin{equation}
\label{equ:t2_3}
    p_{\rvz|\rvu,\rvy}(\rvz|\rvu,\rvy)=\prod_{i=1}^{n} p_{z_i|\rvu,\rvy}(z_i|\rvu,\rvy); \quad p_{\hat{\rvz}|\rvu,\rvy}(\hat{\rvz}|\rvu,\rvy) = \prod_{i=1}^n p_{\hat{z}_i|\rvu,\rvy}(\hat{z}_i|\rvu,\rvy).
\end{equation}
For convenience, we take logarithms on both sides of the Equation(\ref{equ:t2_3}) and further let $q_i:=\log p_{z_i|\rvu,\rvy}, \hat{q}_i:=\log p_{\hat{z}_i|\rvu,\rvy}$. Hence we have:
\begin{equation}
\label{equ:t2_4}
\log p_{\rvz|\rvu,\rvy}(\rvz|\rvu,\rvy)=\sum_{i=1}^n q_i(z_i,\rvu,\rvy);\quad \log p_{\hat{\rvz,\rvy}|\rvu}=\sum_{i=1}^n \hat{q}_i(\hat{z}_i,\rvu,\rvy).
\end{equation}

By combining Equation (\ref{equ:t2_4}) and Equation (\ref{equ:t2_2}), we have:
\begin{equation}
\label{equ:t2_5}
    p_{\rvz|\rvu,\rvy}=p_{h(\hat{\rvz}|\rvu,\rvy)}\Longleftrightarrow p_{\hat{\rvz}|\rvu,\rvy}=p_{\rvz|\rvu,\rvy}|\mJ_{h^{-1}}|\Longleftrightarrow \sum_{i=1}^n q_i({z_i,\rvu,\rvy}) + \log |\mJ_{h^{-1}}| = \sum_{i=1}^n \hat{q}_i(\hat{z}_i,\rvu,\rvy),
\end{equation}
where $\mJ_{h^{-1}}$ are the Jacobian matrix of $h^{-1}$.

Sequentially, we take the first-order derivative with $\hat{z}_j$ on Equation (\ref{equ:t2_5}), where $j\in\{n_1+n_2+n_3+1,\cdots,n\}$, and have
\begin{equation}
\label{equ:t2_app_1_order_der}
    \sum_{i=1}^n\frac{\partial q_i(z_i,\rvu,\rvy)}{\partial z_i}\cdot \frac{\partial z_i}{\partial \hat{z}_j} + \frac{\partial\log |\mJ_{h^{-1}}|}{\partial \hat{z}_j} = \frac{\partial q_j(\hat{z}_j, \rvu,\rvy)}{\partial \hat{z}_j}.
\end{equation}

According to A6, there exist $n_1+n_2+n_3+1$ conbinations of $(\rvu,\rvy)$, so we subtract the Equation (\ref{equ:t2_app_1_order_der}) to $\rvu_k, \rvy_l$ with that corresponds to $\rvu_0,\rvy_0$, and we have:
\begin{equation}
    \sum_{i=1}^{n} \left(\frac{\partial q_i(z_i,u_k,\rvy_l)}{\partial z_i}-\frac{\partial q_i(z_i,u_0,\rvy_0)}{\partial z_i}\right)\cdot \frac{\partial z_i}{\partial \hat{z}_j} = \frac{\partial \hat{q}_j(\hat{z}_j,u_k,\rvy_l)}{\partial \hat{z}_j} - \frac{\partial \hat{q}_j(\hat{z}_j,u_0,\rvy_0)}{\partial \hat{z}_j}.
\end{equation}
Since the distribution of estimated $\hat{\rvz}_j$ does not change across different domains and labels, $\frac{\partial \hat{q}_j(\hat{z}_j,u_k,\rvy_l)}{\partial \hat{z}_j} - \frac{\partial \hat{q}_j(\hat{z}_j,u_0,\rvy_0)}{\partial \hat{z}_j}=0$. Since $\frac{\partial q_i(z_i,u_k,\rvy_l)}{\partial z_i}$ does not change across different domains, $\frac{\partial q_i(z_i,u_k,\rvy_l)}{\partial z_i}=\frac{\partial q_i(z_i,u_0,\rvy_0)}{\partial z_i}$ for $i\in \{1,\cdots, n_1+n_2+n_3\}$. So we have:
\begin{equation}
\sum_{i=1}^{n_1+n_2+n_3} \left(\frac{\partial q_i(z_i,u_k,\rvy_l)}{\partial z_i}-\frac{\partial q_i(z_i,u_0,\rvy_0)}{\partial z_i}\right)\cdot \frac{\partial z_i}{\partial \hat{z}_j} = 0.
\end{equation}

Based on the linear independence assumption (A3), the linear system is a $n \times n$ full-rank system. Therefore, the only solution is $\frac{z_i}{\hat{z}_j}=0$ for $i\in\{1,\cdots,n_1+n_2+n_3\}$ and $j\in\{n_1+n_2+n_3+1,\cdots,n\}$.

Since $h(\cdot)$ is smooth over $\mathcal{Z}$, its Jacobian can be formalized as follows
\begin{equation}
\label{equ:jacobian}
\begin{gathered}
    \mJ_{h}=\begin{bmatrix}
    \begin{array}{c|c|c|c}
        \mJ_{h}^{1,1} & \mJ_{h}^{1,2}& \mJ_{h}^{1,3}& \mJ_{h}^{1,4} \\ \cline{1-4}
        \mJ_{h}^{2,1} & \mJ_{h}^{2,2}& \mJ_{h}^{2,3}& \mJ_{h}^{2,4} \\ \cline{1-4}
        \mJ_{h}^{3,1} & \mJ_{h}^{3,2}& \mJ_{h}^{3,3}& \mJ_{h}^{3,4} \\ \cline{1-4}
        \mJ_{h}^{4,1} & \mJ_{h}^{4,2}& \mJ_{h}^{4,3}& \mJ_{h}^{4,4} \\
    \end{array}
    \end{bmatrix}
\end{gathered}
\end{equation}
where $\mJ^{ij}:=\frac{\partial \rvz_i}{\partial\hat{\rvz}_j}$ and $i,j\in\{1,2,3,4\}$. 

Since $\frac{z_i}{\hat{z}_j}=0$ for $i\in\{1,\cdots,n_1+n_2+n_3\}$ and $j\in\{n_1+n_2+n_3+1,\cdots,n\}$, $\mJ_{h}^{3,4}=0,\mJ_{h}^{2,4}=0,\mJ_{h}^{1,4}=0$.

we take the first-order derivative with $\hat{z}_j$ on Equation (\ref{equ:t2_5}), where $j\in\{n_1+n_2+1,\cdots,n\}$, and have
\begin{equation}
\label{equ:t2_app_1_order_der2}
    \sum_{i=1}^n\frac{\partial q_i(z_i,\rvu,\rvy)}{\partial z_i}\cdot \frac{\partial z_i}{\partial \hat{z}_j} + \frac{\partial\log |\mJ_{h^{-1}}|}{\partial \hat{z}_j} = \frac{\partial q_j(\hat{z}_j, \rvu,\rvy)}{\partial \hat{z}_j}.
\end{equation}
Then we fix the value of $\rvy$ be $\rvy_0$, so there exist $U$ combinations of $(\rvu,\rvy_0)$. We subtract the Equation (\ref{equ:t2_app_1_order_der2}) corresponds to $(\rvu_k,\rvy_0)$ with that corresponds to $(\rvu_0,\rvy_0)$ and have:
\begin{equation}
    \sum_{i=1}^{n} \left(\frac{\partial q_i(z_i,u_k,\rvy_0)}{\partial z_i}-\frac{\partial q_i(z_i,u_0,\rvy_0)}{\partial z_i}\right)\cdot \frac{\partial z_i}{\partial \hat{z}_j} = \frac{\partial \hat{q}_j(\hat{z}_j,u_k,\rvy_0)}{\partial \hat{z}_j} - \frac{\partial \hat{q}_j(\hat{z}_j,u_0,\rvy_0)}{\partial \hat{z}_j}.
\end{equation}
Since the distribution of estimated $\hat{\rvz}_j$ does not change across different domains, $\frac{\partial \hat{q}_j(\hat{z}_j,u_k,\rvy_0)}{\partial \hat{z}_j} - \frac{\partial \hat{q}_j(\hat{z}_j,u_0,\rvy_0)}{\partial \hat{z}_j}=0$. Since $\frac{\partial q_i(z_i,u_k,\rvy_0)}{\partial z_i}$ does not change across different domains, $\frac{\partial q_i(z_i,u_k,\rvy_0)}{\partial z_i}=\frac{\partial q_i(z_i,u_0,\rvy_0)}{\partial z_i}$ for $i\in \{1,\cdots,n_1+n_2\}$. So we have:
\begin{equation}
\sum_{i=1}^{n_1+n_2} \left(\frac{\partial q_i(z_i,u_k,\rvy_0)}{\partial z_i}-\frac{\partial q_i(z_i,u_0,\rvy_0)}{\partial z_i}\right)\cdot \frac{\partial z_i}{\partial \hat{z}_j} = 0.
\end{equation}
Based on the linear independence assumption (A3), the linear system is a $n \times n$ full-rank system. Therefore, the only solution is $\frac{z_i}{\hat{z}_j}=0$ for $i\in\{1,\cdots,n_1+n_2\}$ and $j\in\{n_1+n_2+1,\cdots,n\}$. Combining Equation (\ref{equ:jacobian}), we can find that $\mJ_h^{1,3}=0,\mJ_h^{1,4}=0,\mJ_h^{2,3}=0$, and $\mJ_h^{2,4}=0$.

Similarly, we let $j\in\{1,\cdots,n_1\}\bigcup\{n_1+n_2+n_3+1,\cdots,n\}$ and have:
\begin{equation}
\label{equ:t2_app_1_order_der3}
    \sum_{i=1}^n\frac{\partial q_i(z_i,\rvu,\rvy)}{\partial z_i}\cdot \frac{\partial z_i}{\partial \hat{z}_j} + \frac{\partial\log |\mJ_{h^{-1}}|}{\partial \hat{z}_j} = \frac{\partial q_j(\hat{z}_j, \rvu,\rvy)}{\partial \hat{z}_j}.
\end{equation}

Then fix the value of $\rvu$ be $\rvu_0$, so there exist $C$ combinations of $(\rvu_0,\rvy_l)$. We subtract the Equation (\ref{equ:t2_app_1_order_der3}) corresponds to $(\rvu_0,\rvy_l)$ with that corresponds to $(\rvu_0,\rvy_0)$ and have:
\begin{equation}
    \sum_{i=n_1+1}^{n_1+n_2+n_3} \left(\frac{\partial q_i(z_i,u_0,\rvy_l)}{\partial z_i}-\frac{\partial q_i(z_i,u_0,\rvy_0)}{\partial z_i}\right)\cdot \frac{\partial z_i}{\partial \hat{z}_j} = \frac{\partial \hat{q}_j(\hat{z}_j,u_0,\rvy_l)}{\partial \hat{z}_j} - \frac{\partial \hat{q}_j(\hat{z}_j,u_0,\rvy_0)}{\partial \hat{z}_j}.
\end{equation}
Based on the linear independence assumption (A3), the linear system is a $n \times n$ full-rank system. Therefore, the only solution is $\frac{z_i}{\hat{z}_j}=0$ for $i\in\{n_1+1,\cdots,n_1+n_2+n_3\}$ and $j\in\{1,\cdots,n_1\}\bigcup\{n_1+n_2+n_3+1,\cdots,n\}$. Combining Equation (\ref{equ:jacobian}), we can find that $\mJ_h^{2,1}=0,\mJ_h^{2,4}=0,\mJ_h^{3,1}=0$, and $\mJ_h^{3,4}=0$.

In summary, Equation (\ref{equ:jacobian}) can be written as follows
\begin{equation}
\label{equ:jacobian2}
\begin{gathered}
    \mJ_{h}=\begin{bmatrix}
    \begin{array}{c|c|c|c}
        \mJ_{h}^{1,1} & \mJ_{h}^{1,2}& \mJ_{h}^{1,3}=0& \mJ_{h}^{1,4}=0 \\ \cline{1-4}
        \mJ_{h}^{2,1}=0 & \mJ_{h}^{2,2}& \mJ_{h}^{2,3}=0& \mJ_{h}^{2,4}=0 \\ \cline{1-4}
        \mJ_{h}^{3,1}=0 & \mJ_{h}^{3,2}& \mJ_{h}^{3,3}& \mJ_{h}^{3,4}=0 \\ \cline{1-4}
        \mJ_{h}^{4,1} & \mJ_{h}^{4,2}& \mJ_{h}^{4,3}& \mJ_{h}^{4,4} \\
    \end{array}
    \end{bmatrix}.
\end{gathered}
\end{equation}
Since  $h(\cdot)$ is invertible, $\mJ_{h}$ is a full-rank matrix. Therefore, for each $z_{2,i}, i\in\{n_1+1,\cdots,n_1+n_2\}$, there exists a $h_{2,i}$ such that  $z_{2,i}=h_i(\hat{\rvz}_2)$. Moreover, for each $z_{1,i}, i\in\{1,\cdots,n_1+1\}$, there exists a $h_{1,i}$ such that  $z_{1,i}=h_{1,i}(\hat{\rvz}_1,\hat{\rvz}_2)$. And for each $z_{3,i}, i\in\{n_1+n_2+1,\cdots,n_1+n_2+n_3\}$, there exists a $h_{3,i}$ such that  $z_{3,i}=h_{3,i}(\hat{\rvz}_2,\hat{\rvz}_3)$.
\end{proof}

\subsection{Proof of Blockwise Identification}
\label{app:proof_block}
\begin{lemma}
\label{lamma}\cite{kong2022partial}
Following the data generation process in Section 2.1 and the assumptions A4-A6 in Theorem \ref{the1}, we further make the following assumption:
\begin{itemize}[leftmargin=*,itemsep=-3pt]
    \item A7 (\underline{Domain Variability}: For any set $A_{\rvz} \subseteq \mathcal{Z}$) with the following two properties: 1) $A_{\rvz}$ has nonzero probability measure, i.e. $\mathbb{P}[\rvz\in A_{\rvz}|\{\rvu=\rvu', \rvy=\rvy'\}]>0$ for any $\rvu'\in \mathcal{U}$ and $\rvy'\in \mathcal{Y}$. 2) $A_{\rvz}$ cannot be expressed as $B_{\rvz_4} \times \mathcal{Z}_1\times \mathcal{Z}_2\times \mathcal{Z}_3$ for any $B_{\rvz_4}\subset \mathcal{Z}_4$.
\end{itemize}
$\exists \rvu_1, \rvu_2 \in \mathcal{U}$ and $\rvy_1,\rvy_2 \in \mathcal{Y}$, such that $\int_{\rvz\in \mathcal{A}_{\rvz}}p_{\rvz|\rvu_1,\rvy_1}d\rvz \neq \int_{\rvz\in \mathcal{A}_{\rvz}}p_{\rvz|\rvu_2,\rvy_2}d\rvz$. By modeling the data generation process in Section 2.1, the $\rvz_4$ is block-wise identifiable. 
\end{lemma}
\begin{proof}
We divide the proof into four steps for better understanding.

In Step 1, we leverage the properties of the data generation process and the marginal distribution matching condition to express the marginal invariance with the indeterminacy transformation $\overline{h}:\mathcal{Z}\rightarrow\mathcal{Z}$ between the estimated and the ground-truth latent variables. The introduction of $\overline{h}(\cdot)$ allows us to formalize the block-identifiability condition.

In Step 2 and Step 3, we show that the estimated $\hat{\rvz}_4$ does not depend on the ground-truth changing variables, i.e., $\rvz_1,\rvz_2,\rvz_3$, that is, $\overline{h}_4(\rvz)$ does not depend on the input $\{\rvz_1,\rvz_2,\rvz_3\}$. To this end, in Step 2, we derive its equivalent statements which can ease the rest of the proof and avert technical issues (e.g. sets of zero probability measures). In Step 3, we prove the equivalent statement by contradiction. Specifically, we show that if $\hat{\rvz}_4$ depends of $\rvz_1,\rvz_2,\rvz_3$, the invariance derived in Step 1 would break.

In Step 4, we use the conclusion in Step 3, the smooth and bijective properties of $h(\cdot)$, and the conclusion in Corollary 1.1, to show the invertibility of the indeterminacy function between the ground-truth $\rvz_4$ and estimated $\hat{\rvz}_4$, i.e. the mapping $\hat{\rvz}_4=\overline{h}_4(\rvz_4)$ being invertible.

\textbf{Step 1.} As the data generation process in Section 2.1 establishes the independence between the generation process $\hat{\rvz}_4 \sim p_{\hat{\rvz}_4}$ and $\rvu$ it follows that for any $A_{\rvz_4} \subseteq \mathcal{Z}_4$, we let $n_s=n_1+n_2+n_3$, then we have:
\begin{equation}
\label{equ:th3_1}
\begin{split}
    \forall \rvu_1,\rvu_2 \in \mathcal{U}, \rvy_1, \rvy_2 \in \mathcal{Y}\quad\quad\quad\quad\quad\quad\quad&\\ 
    \mathbb{P}\left[\{\hat{g}^{-1}_{n_s:n}(\hat{\rvx})\in A_{\rvz_4}\}|\{\rvu=\rvu_1, \rvy=\rvy_1\}\right]&=\mathbb{P}\left[\{\hat{g}^{-1}_{n_s:n}(\hat{\rvx})\in A_{\rvz_4}\}|\{\rvu=\rvu_2,\rvy=\rvy_2\}\right]\\ &\Longleftrightarrow\\ \forall \rvu_1,\rvu_2 \in \mathcal{U}, \rvy_1, \rvy_2 \in \mathcal{Y}\quad\quad\quad\quad\quad\quad\quad&\\ 
\mathbb{P}\left[\hat{x}\in(\hat{g}_{n_s:n}^{-1})^{-1}(A_{\rvz_4})|\{\rvu=\rvu_1,\rvy=\rvy_1\}\right]&=\mathbb{P}\left[\hat{x}\in(\hat{g}_{n_s:n}^{-1})^{-1}(A_{\rvz_4})|\{\rvu=\rvu_2,\rvy=\rvy_2\}\right],
\end{split}
\end{equation}
where $\hat{g}^{-1}_{n_s:n}:\mathcal{X}\rightarrow\mathcal{Z}_4$ denotes the estimated transformation from the observation to the $\rvz_4$ latent variables; and $(\hat{g}^{-1}_{n_s:n})^{-1}(A_{\rvz_4})\subseteq \mathcal{X}$ is the pre-image set of $A_{\rvz_4}$, that is , the set of estimated observations $\hat{\rvx}$ originating from $\rvz_4$ in $A_{\rvz_4}$.

Because of the matching observation distributions between the estimated model and the true model, the relation in the Equation (\ref{equ:th3_1}) can be extended to observation $\rvx$ from the true generating process, i.e.,
\begin{equation}
\label{equ:th3_2}
\begin{split}
    \mathbb{P}\left[\{\rvx\in(\hat{g}^{-1}_{n_s:n})^{-1}(A_{\rvz_4})\}|\{\rvu=\rvu_1,\rvy=\rvy_1\}\right]&=\mathbb{P}\left[\{\rvx\in(\hat{g}^{-1}_{n_s:n})^{-1}(A_{\rvz_4})\}|\{\rvu=\rvu_2,\rvy=\rvy_2\}\right]\\ &\Longleftrightarrow\\ \mathbb{P}\left[\{\hat{g}_{n_s:n}^{-1}(\rvx)\in A_{\rvz_4}\}|\rvu=\rvu_1,\rvy=\rvy_1\right]&=\mathbb{P}\left[\{\hat{g}_{n_s:n}^{-1}(\rvx)\in A_{\rvz_4}\}|\rvu=\rvu_2,\rvy=\rvy_2\right].
\end{split}
\end{equation}
Since $g$ and $\hat{g}$ are smooth and injective, there exists a smooth and injective $\overline{h}=\hat{g}^{-1}\circ g:\mathcal{Z}\rightarrow\mathcal{Z}$. We note that by definition $\overline{h}=h$ where $h$ is introduced in the proof of Theorem \ref{the1}. Expressing $\hat{g}^{-1}=\overline{h}\circ g^{-1}$ and $\overline{h}_4(\cdot):=\overline{h}_{n_s:n}(\cdot):\mathcal{Z}\rightarrow\mathcal{Z}_4$ in Equation (\ref{equ:th3_2}) yields
\begin{equation}
\label{equ:th3_3}
\begin{split}
\mathbb{P}\left[\{\overline{h}_4(\rvz)\in A_{\rvz_4}\}|\{\rvu=\rvu_1,\rvy=\rvy_1\}\right]&=\mathbb{P}\left[\{\overline{h}_4(\rvz)\in A_{\rvz_4}\}|\{\rvu=\rvu_2,\rvy=\rvy_2\}\right] \\ &\Longleftrightarrow \\\mathbb{P}\left[\{\rvz\in \overline{h}_4^{-1}(A_{\rvz_4})\}|\{\rvu=\rvu_1,\rvy=\rvy_1\}\right]&=\mathbb{P}\left[\{\rvz\in \overline{h}_4^{-1}(A_{\rvz_4})\}|\{\rvu=\rvu_2,\rvy=\rvy_2\}\right]\\ &\Longleftrightarrow \\ \int_{\rvz\in\overline{h}_4^{-1}(A_{\rvz_4})}p_{\rvz|\rvu,\rvy}(\rvz|\rvu_1,\rvy_1)d\rvz&=\int_{\rvz\in \overline{h}_4^{-1}(A_{\rvz_4})}p_{\rvz|\rvu,\rvy}(\rvz|\rvu_2,\rvy_2)d\rvz,
\end{split}
\end{equation}
where $\overline{h}_4^{-1}(A_{\rvz_4})=\{\rvz\in \mathcal{Z}:\overline{h}_4(\rvz)\in A_{\rvz_4}\}$ is the pre-image of $A_{\rvz_4}$, i.e., those latent variables containing $\rvz_4$ in $A_{\rvz_4}$ after the indeterminacy transformation $h$.

Based on the proposed generation process in Section 2.1, we rewrite Equation (\ref{equ:th3_3}) as follows:
\begin{equation}
\label{equ:th3_4}
\begin{split}
    &\forall A_{\rvz_4} \subseteq \mathcal{Z}_4, \\ &\int_{[\rvz_1^{\top}, \rvz_2^{\top}, \rvz_3^{\top}, \rvz_4^{\top}]^{\top} \in \overline{h}_{4}^{-1}(A_{\rvz_4})} p_{\rvz_4}(\rvz_4)(p_{\rvz_1,\rvz_2,\rvz_3|\rvu,\rvy}(\rvz_1,\rvz_2,\rvz_3|\rvu_1,\rvy_1)\\&-p_{\rvz_1,\rvz_2,\rvz_3|\rvu,\rvy}(\rvz_1,\rvz_2,\rvz_3|\rvu_2,\rvy_2))d\rvz_1d\rvz_2d\rvz_3d\rvz_4=0
\end{split}
\end{equation}

\textbf{Step 2.}In order to show the block-identifiability of $\rvz_4$, we would like to prove that $\rvz_c:=\overline{h}([\rvz_1^{\top},\rvz_2^{\top},\rvz_3^{\top},\rvz_4^{\top}]^{\top})$ does not depend on $\rvz_{1:n_s}$. To this end, we first develop one equivalent statement (i.e., State 3 below) and prove it in a later step instead. By doing so, we are able to leverage the full-support density function assumption to avert technical issues.
\begin{itemize}
    \item Statement 1: $\overline{h}_4([\rvz_1^{\top},\rvz_2^{\top},\rvz_3^{\top},\rvz_4^{\top}]^{\top})$ does not depend on $\rvz_{1:n_s}$
    \item Statement 2: $\forall \rvz_4\in \mathcal{Z}_4$, it follows that $\overline{h}_4^{-1}=B_{\rvz_4}\times\mathcal{Z}_1\times\mathcal{Z}_2\times\mathcal{Z}_3$ where $B_{\rvz_4} \neq \emptyset$ and $B_{\rvz_4}\subseteq\mathcal{Z}_4$. 
    \item Statement 3: $\forall \rvz_4 \in \mathcal{Z}_4, r\in \mathbb{R}^{+}$, it follows that $\overline{h}_4^{-1}(\mathcal{B}_r(\rvz_4))=B_{\rvz_4}^{+}\times\mathcal{Z}_1\times\mathcal{Z}_2\times\mathcal{Z}_3$ where $\mathcal{B}_r(\rvz_4):=\{\rvz_4'\in\mathcal{Z}_4:||\rvz_4'-\rvz_4||^2<r\}, B_{\rvz_4}^+\neq \emptyset$, and $B_{\rvz_4}^+ \subseteq \mathcal{Z}_4$.
\end{itemize}
Statement 2 is a mathematical formulation of Statement 1. Statement 3 generalizes singletons $\rvz_4$ in Statement 2 to open, non-empty balls $\mathcal{B}_r(\rvz_4)$. Later, we use Statement 3 in Step 3 to show the contraction to Equation (\ref{equ:th3_4}).

Leveraging the continuity of $\overline{h}_4(\cdot)$, we can show the equivalence between Statement 2 and Statement 3 as follows. We first show that Statement 2 implies Statement 3. $\forall \rvz_4, r\in\mathbb{R}^+, \overline{h}_c^{-1}(\mathcal{B}(\rvz_4))=\bigcup_{\rvz'_4\in \mathcal{B}_r(\rvz_4)}h_4^{-1}(\rvz'_4)$. Statement 2 indicates that every participating sets in the union satisfies $h^{-1}_4(\rvz'_4)=B'_{\rvz_4}\times\mathcal{Z}_1\times\mathcal{Z}_2\times\mathcal{Z}_3$, thus the union $\overline{h}_c^{-1}(\mathcal{B}_r(\rvz_4))$ also satisfies this property, which is Statement 3.

Then, we show that Statement 3 implies Statement 2 by contradiction. Suppose that Statement 2 is false, then $\exists \hat{\rvz}_4\in\mathcal{Z}_4$ such that there exist $\hat{\rvz}_4^{B}\in\{\rvz_{n_s:n}:\rvz\in\overline{h}_4^{-1}(\hat{\rvz}_4)\}$ and $\hat{\rvz}_{n_s}^B \in \mathcal{Z}_{n_s}$ resulting in $\overline{h}_4(\hat{\rvz}^B)\neq \hat{\rvz}_4$ where $\hat{\rvz}^B=[(\hat{z}^B_4)^{\top},(\hat{\rvz}_{n_s}^B)^{\top}]^{\top}$. As $\overline{h}_4(\cdot)$ is continuous, there exists $\hat{r}\in\mathbb{R}^+$ such that $\overline{h}_4(\hat{\rvz}^B)\notin \mathcal{B}_{\hat{r}}(\hat{\rvz}_4)$. That is, $\hat{\rvz}^B \notin h_4^{-1}(\mathcal{B}_{\hat{r}}(\hat{\rvz}_4))$. Also, Statement 4 suggests that $h_4^{-1}(\mathcal{B}_{\hat{r}}(\hat{\rvz}_c))=\hat{B}_{\rvz_4}\times \mathcal{Z}_{n_s}$. By definition of $\hat{\rvz}^B$, it is clear that $\hat{\rvz}^B_{n_s:n}\in \hat{B}_{\rvz_4}$. The fact that $\hat{\rvz}^B\notin h^{-1}_4(\mathcal{B}_{\hat{r}}(\hat{\rvz}_4))$ contradicts Statement 3. Therefore, Statement 2 is true under the premise of Statement 3. We have shown that Statement 3 implies Statement 2. In summary, Statement 2 and Statement 3 are equivalent, and therefore proving Statement 3 suffices to show Statement 1.

\textbf{Step 3.} In this step, we prove State 3 by contradiction. Intuitively, we show that if $\overline{h}_4(\cdot)$ depended on $\hat{\rvz}_{1},\hat{\rvz}_{2},\hat{\rvz}_{3}$, the preimage $\overline{h}_4^{-1}(\mathcal{B}_r(\rvz_4))$ could be partitioned into two parts (i.e. $B^*_{\rvz}$ and $\overline{h}_4^{-1}(A^*_{\rvz_4})\textbackslash B^*_{\rvz}$ defined below). The dependency between $\overline{h}_4(\cdot)$ and $\hat{\rvz}_4$ is captured by $B_{\rvz}^*$, which would not emerge otherwise. In contrast, $\overline{h}_{4}^{-1}\textbackslash B_{\rvz}^*$ also exists when $\overline{h}_4(\cdot)$ does not depend on $\hat{\rvz}_1,\hat{\rvz}_2,\hat{\rvz}_3$. We evaluate the invariance relation Equation (\ref{equ:th3_4}) and show that the integral over $\overline{h}_4^{-1}(A^*_{\rvz_4}) \textbackslash B^*_{\rvz}$ is always 0, however, the integral over $B^*_{\rvz}$ is necessarily non-zero, which leads to the contraction with Equation (\ref{equ:th3_4}) and thus show the $\overline{h}_4(\cdot)$ cannot depend on $\hat{\rvz}_1,\hat{\rvz}_2,\hat{\rvz}_3,$

First, note that because $\mathcal{B}_r(\rvz_4)$ is open and $\overline{h}_4(\cdot)$ is continuous, the pre-image $\overline{h}_4^{-1}(\mathcal{B}_r(\rvz_4))$ is open. In addition, the continuity of $h(\cdot)$ and the matched observation distributions $\forall \rvu'\in \mathcal{U}, \mathbb{P}[\{\rvx\in A_{\rvx}\}|\{\rvu=\rvu',\rvy=
\rvy'\}]=\mathbb{P}[\{\hat{\rvx}\in A_{\rvx}\}|\{\rvu=\rvu',\rvy=
\rvy'\}]$ lead to $h(\cdot)$ being bijection as shown in \cite{klindt2020towards}, which implies that $\overline{h}_4^{-1}(\mathcal{B}_r(\rvz_4))$ is non-empty. Hence, $\overline{h}_4^{-1}(\mathcal{B}_r(\rvz_4))$ is both non-empty and open. Suppose that $\exists A_{\rvz_4}^*:=\mathcal{B}_{r^*}(\rvz_4^*$ where $\rvz_4^* \in \mathcal{Z}_4, r^*\in\mathbb{R}^+$, such that $B_{\rvz}^*=\{ \rvz \in \mathcal{Z}: \rvz\in \overline{h}_c^{-1}(A_{\rvz_4}^*),\{\rvz_{n_s:n}\}\times\mathcal{Z}_{n_s} \nsubseteq  \overline{h}_4^{-1}(A_{\rvz_4}^*)  \} \neq \emptyset$. Intuitively, $B_{\rvz}^*$ contains the partition of the pre-image $\overline{h}_4^{*}(A_{\rvz}^*)$ that the style part $\rvz_{1:n_s}$ can not take on any value in $\mathcal{Z}_1,\mathcal{Z}_2,\mathcal{Z}_3$. Only certain values of the style part were able to produce specific outputs of indeterminacy $\overline{h}_4(\cdot)$. Clearly, this would suggest that $\overline{h}_4(\cdot)$ depends on $\rvz_4$.
To show contraction with Equation (\ref{equ:th3_4}), we evaluate the LHS of Equation (\ref{equ:th3_4}) with such a $A_{\rvz_4}^*$:
\begin{equation}
\small
\begin{split}
&\int_{[\rvz_1^{\top}, \rvz_2^{\top}, \rvz_3^{\top}, \rvz_4^{\top}]^{\top} \in \overline{h}_{4}^{-1}(A_{\rvz_4}^*)} P_{\rvz_4}(\rvz_4)\left(p_{\rvz_1,\rvz_2,\rvz_3|\rvu,\rvy}(\rvz_1,\rvz_2,\rvz_3|\rvu_1,\rvy_1)-p_{\rvz_1,\rvz_2,\rvz_3|\rvz,\rvy}(\rvz_1,\rvz_2,\rvz_3|\rvu_2,\rvy_2)\right)d\rvz_1d\rvz_2d\rvz_3d\rvz_4\\=&\underbrace{\int_{[\rvz_1^{\top}, \rvz_2^{\top}, \rvz_3^{\top}, \rvz_4^{\top}]^{\top} \in \overline{h}_{4}^{-1}(A_{\rvz_4}^*)\textbackslash B_{\rvz}^*}P_{\rvz_4}(\rvz_4)\left(p_{\rvz_1,\rvz_2,\rvz_3|\rvu,\rvy}(\rvz_1,\rvz_2,\rvz_3|\rvu_1,\rvy_1)-p_{\rvz_1,\rvz_2,\rvz_3|\rvu,\rvy}(\rvz_1,\rvz_2,\rvz_3|\rvu_2,\rvy_2)\right)d\rvz_1d\rvz_2d\rvz_3d\rvz_4}_{T_1} \\+& \underbrace{\int_{[\rvz_1^{\top}, \rvz_2^{\top}, \rvz_3^{\top}, \rvz_4^{\top}]^{\top} \in B_{\rvz}^*}P_{\rvz_4}(\rvz_4)\left(p_{\rvz_1,\rvz_2,\rvz_3|\rvu,\rvy}(\rvz_1,\rvz_2,\rvz_3|\rvu_1,\rvy_1)-p_{\rvz_1,\rvz_2,\rvz_3|\rvu,\rvy}(\rvz_1,\rvz_2,\rvz_3|\rvu_2,\rvy_2)\right)d\rvz_1d\rvz_2d\rvz_3d\rvz_4}_{T_2}
\end{split}
\end{equation}
We first look at the value of $T_1$. When $\overline{h}_{4}^{-1}(A_{\rvz_4}^*)\textbackslash B_{\rvz}^*=\emptyset$, $T_1$ evaluates to 0. Otherwise, by definition, we can rewrite $\overline{h}_{4}^{-1}(A_{\rvz_4}^*)\textbackslash B_{\rvz}^*$ as $C^*_{\rvz_4}\times \mathcal{Z}_1\times \mathcal{Z}_2\times \mathcal{Z}_3$ where $C^*_{\rvz_4}\subset \mathcal{Z}_4$. With this expression, it follows that
\begin{equation}
\small
\begin{split}
&\int_{[\rvz_1^{\top}, \rvz_2^{\top}, \rvz_3^{\top}, \rvz_4^{\top}]^{\top} \in C^*_{C_{\rvz}^*}} P_{\rvz_4}(\rvz_4)\left(p_{\rvz_1,\rvz_2,\rvz_3|\rvu,\rvy}(\rvz_1,\rvz_2,\rvz_3|\rvu_1,\rvy_1)-p_{\rvz_1,\rvz_2,\rvz_3|\rvu,\rvy}(\rvz_1,\rvz_2,\rvz_3|\rvu_2,\rvy_2)\right)d\rvz_1d\rvz_2d\rvz_3d\rvz_4 \\=& \int_{\rvz_4\in C_{\rvz_4}^*}p_{\rvz_4}(\rvz_4) \int_{\rvz_1 \in \mathcal{Z}_1}\int_{\rvz_2 \in \mathcal{Z}_2}\int_{\rvz_3 \in \mathcal{Z}_3}(p_{\rvz_1,\rvz_2,\rvz_3|\rvu,\rvy}(\rvz_1,\rvz_2,\rvz_3|\rvu_1,\rvy_1)-p_{\rvz_1,\rvz_2,\rvz_3|\rvu,\rvy}(\rvz_1,\rvz_2,\rvz_3|\rvu_2,\rvy_2))d\rvz_1d\rvz_2d\rvz_3d\rvz_4\\=& \int_{\rvz_4 \in C_{\rvz_4^*}}p_{\rvz_4}(\rvz_4)(1-1)d_{\rvz_4}=0.
\end{split}
\end{equation}
Therefore, in both cases $T_1$ evaluates to 0 for $A_{\rvz_4}^*$.

Now, we address $T_2$. As discuss above, $\overline{h}_4^{-1}(A_{\rvz_4}^*)$ is open and non-empty. Because of the continuity of $\overline{h}_4(\cdot)$, $\forall \rvz_{B} \in B_{\rvz}^*$, there exists $r(\rvz_B)\in \mathbb{R}^+$ such that $\mathcal{B}_{r(\rvz_B)}(\rvz_B) \subseteq B_{\rvz}^*$. As $p_{\rvz|\rvu,\rvy}>0$ over $(\rvu,\rvz,\rvy)$, we have $\mathbb{P}[\{\rvz\in B_{\rvz}^*\}|\{\rvu=\rvu',\rvy=\rvy'\}]\geq \mathbb{P}[\{\rvz\in B_{r(\rvz_B)}(\rvz_B)\}|\{\rvu=\rvu' ,\rvy=\rvy'\}]>0$ for any $\rvz'\in\mathcal{U}, \rvy\in \mathcal{Y}$. Assumption A7 indicates that $\exists \rvu_1^*,\rvu_2^*$, such that 
\begin{equation}
\begin{split}
    T_2:=\int_{[\rvz_1^{\top}, \rvz_2^{\top}, \rvz_3^{\top}, \rvz_4^{\top}]^{\top} \in B_{\rvz}^*}P_{\rvz_4}(\rvz_4)&\left(p_{\rvz_1,\rvz_2,\rvz_3|\rvu,\rvy}(\rvz_1,\rvz_2,\rvz_3|\rvu_1,\rvy_1)\right.\\&-\left.p_{\rvz_1,\rvz_2,\rvz_3|\rvu,\rvy}(\rvz_1,\rvz_2,\rvz_3|\rvu_2,\rvy_2)\right)d\rvz_1d\rvz_2d\rvz_3d\rvz_4 \neq 0.
\end{split}
\end{equation}
Therefore, for such $A_{\rvz_4}^*$, we would have $T_1+T_2\neq 0$ which leads to contradiction with Equation (\ref{equ:th3_4}). We have proved by contradiction that Statement 3 is true and hence Statement 1 holds, that is, $\overline{h}_4(\cdot)$ does not depend on the changing variables $\rvz_1,\rvz_2,\rvz_3$.

\textbf{Step 4.}With the knowledge that $\overline{h}_4(\cdot)$ does not depend on the changing variables $\rvz_1,\rvz_2,\rvz_3$, we now show that there exists an invertible mapping between the true $\rvz_4$ and the estimated $\rvz_4$.

As $\overline{h}(\cdot)$ is smooth over $\mathcal{Z}$, its Jacobian can written as:
\begin{equation}
\label{equ:jacobian_block}
\begin{gathered}
    \mJ_{\overline{h}}=\begin{bmatrix}
    \begin{array}{c|c|c|c}
        \mJ_{\overline{h}}^{1,1} & \mJ_{\overline{h}}^{1,2}& \mJ_{\overline{h}}^{1,3}& \mJ_{\overline{h}}^{1,4} \\ \cline{1-4}
        \mJ_{\overline{h}}^{2,1} & \mJ_{\overline{h}}^{2,2}& \mJ_{\overline{h}}^{2,3}& \mJ_{\overline{h}}^{2,4} \\ \cline{1-4}
        \mJ_{\overline{h}}^{3,1} & \mJ_{\overline{h}}^{3,2}& \mJ_{\overline{h}}^{3,3}& \mJ_{\overline{h}}^{3,4} \\ \cline{1-4}
        \mJ_{\overline{h}}^{4,1} & \mJ_{\overline{h}}^{4,2}& \mJ_{\overline{h}}^{4,3}& \mJ_{\overline{h}}^{4,4} \\
    \end{array}
    \end{bmatrix},
\end{gathered}
\end{equation}
in which $\mJ_{\overline{h}}^{i,j}$ denotes $\frac{\partial \hat{\rvz}_i}{\partial \hat{\rvz}_j}, i,j \in \{1,2,3,4\}$; and we use notation $\hat{\rvz}_4=\overline{h}(\rvz)_{n_s:n}$, $\hat{\rvz}_1=\overline{h}(\rvz)_{1:n_1}$, $\hat{\rvz}_2=\overline{h}(\rvz)_{n_1+1:n_2}$, $\hat{\rvz}_3=\overline{h}(\rvz)_{n_1+n_2+1:n_3}$. As we have shown that $\hat{\rvz}_4$ does not depend on the changing variables $\rvz_1,\rvz_2,\rvz_3$, if follows $\mJ_{\overline{h}}^{4,1}=0, \mJ_{\overline{h}}^{4,2}=0, \mJ_{\overline{h}}^{4,3}=0$. On the other hand, as $h(\cdot)$ is invertible over $\mathcal{Z}$, $\mJ_{\overline{h}}$ is non-singular. Therefore, $\mJ_{\overline{h}}^{4,4}$ must be non-singular. We note that $\mJ_{\overline{h}}^{4,4}$ is the Jacobian of the function $\overline{h}'_4:=\overline{h}_c(\rvz): \mathcal{Z}_4\rightarrow\mathcal{Z}_4$, which takes only the $\rvz_4$ of the input $\rvz$ into $\overline{h}_4$. According to Corollary 1.1, we also find that $\mJ_{\overline{h}}^{1,4}=0,\mJ_{\overline{h}}^{2,4}=0,\mJ_{\overline{h}}^{3,4}=0$. Together with the invertibility of $\overline{h}$, we can conclude that $\overline{h}'_4$ is invertible. Therefore, there exists an invertible function $\overline{h}_4'$ between the estimated and the true variables such that $\hat{\rvz}_4=\overline{h}_4'(\rvz_4)$, which concludes the proof that $\rvz_4$ is block identifiable via $\hat{g}^{-1}(\cdot)$.


\end{proof}

\section{Implementation Details}
%
\label{app:imple}

\begin{table*}[!ht]
    \centering
    \label{fig:imple}
    \caption{Implementation details of the SIG model in different datasets.}
    \begin{tabular}{l|cccc}
    \toprule
        Datasets & Office-Home & ImageCLEF & PACS & DomainNet \\ \hline
        Encoder & 2-layers MLPs & 2-layers MLPs & 2-layers MLPs & 1-layers MLPs \\ \hline
        Decoder & 2-layers MLPs & 2-layers MLPs & 2-layers MLPs & 2-layers MLPs \\ \hline
        learning rate & 0.008 & 0.01 & 0.01 & 0.001 \\ \hline
        $\alpha$ & 1.00E-05 & 1.00E-05 & 1.00E-05 & 1.00E-05 \\ \hline
        $\beta$ & 0.1 & 0.1 & 0.1 & 0.1 \\ \hline
        $\rvz_1$ dimension & 2 & 4 & 2 & 2 \\ \hline
        $\rvz_2$ dimension & 128 & 128 & 60 & 2048 \\ \hline
        $\rvz_3$ dimension & 128 & 10 & 24 & 32 \\ \hline
        $\rvz_4$ dimension & 10 & 4 & 2 & 2 \\ \hline
        Optimizer & SGD & SGD & SGD & SGD \\ \hline
        Momentum & 0.9 & 0.9 & 0.9 & 0.9 \\ \hline
        batch size & 32 & 32 & 32 & 100 \\ \hline
        backbone & ResNet50 & ResNet50 & ResNet18 & ResNet101-based CAN \\ \hline
    \end{tabular}
\end{table*}

The implementation details of the proposed SIG model are shown in Table 1. For Office-Home and ImageCLEF datasets, we employ the pre-trained ResNet50 as the backbone networks. For the PACS dataset, we use the pre-trained ResNet18 as the backbone network. It is noted that we employ a ResNet101-based cross-attention network (CAN) as the backbone network, which is shown in Figure \ref{fig:can}. In CAN, we inject a cross-attention module into each block of the pre-trained ResNet. Technologically, we use the input feature (e.g. $f_1$ in Figure \ref{fig:can}) and the domain index to calculate the weights $\rvw_c$. Sequentially, we take $\rvw_c\odot f_1$ as the input of the pre-trained ResNet Layers and obtain the output of each block.

\begin{figure}[t]
  \centering
\includegraphics[width=1.0\columnwidth]{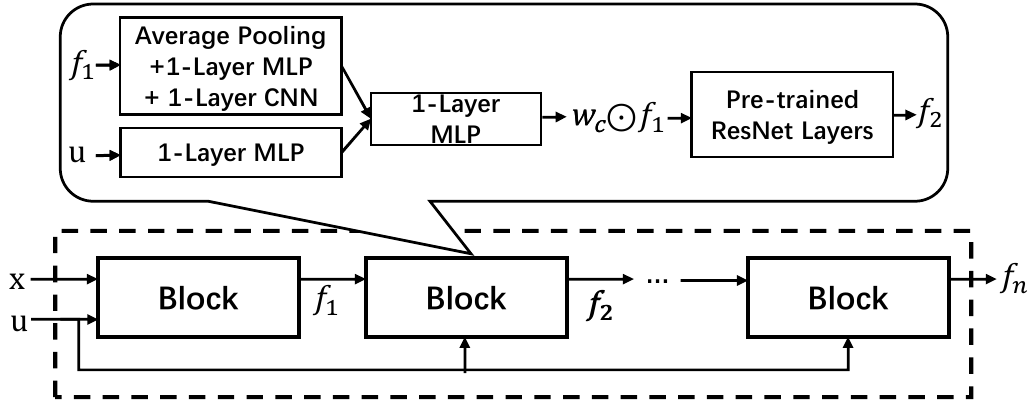}
  \caption{A illustrate framework of the ResNet101-based cross-attention networks (CAN). In each block of the ResNet101, we use the domain information and the inputs of each block to calculate the weights $\rvw_c$ of each dimension of the feature, which dynamically selects the most relevant features.} 
  \label{fig:can} 
\end{figure}

\section{Experiments}
\subsection{Simulation Data Experiments}
\label{app:simu_exp}

We provide more details for the simulation experiments. First, we introduce the details of model architecture for simulation experiments. Second, we further provide the training hyper-parameters.

\subsubsection{Model Architecture.}
For the model architecture of our simulation experiments, the variational auto-encoder (VAE) encoder and decoder are 1-layer MLPs with a hidden dimension of 200, a ReLU activation function, a batch normalization layer, and a dropout layer.

\subsubsection{Training Hyper-parameters.}
We use an SGD optimizer with a momentum of 0.9 to train VAE models with 50 epochs. We also use a learning rate of 0.0035 with a batch size of 768. For the VAE training, we set the hyper-parameters of the KL loss to 1.

\subsection{Real-world Data Experiments}\label{app:real_exp}

We provide implementation details of real-world data experiments. First, we provide detailed descriptions of Office-Home, ImageCLEF, PACS, and DomainNet datasets. Second, we show more experiment results, including more baselines, the mean, and the standard deviation of the results.

\subsubsection{Dataset Description}
\textbf{Office-Home} is a benchmark dataset with 4 domains, where each domain contains 65 categories. These four domains are shown as follows: Art contains artistic images in the form of sketches, paintings, ornamentation, etc.; Clipart contains the collection of clipart images; Product contains images of objects without a background and Real-World contains images of objects captured with a regular camera.
\textbf{ImageCLEF} is a standard domain adaptation benchmark dataset for image classification, consisting of three domains: Caltech-256(C), ImageNet
ILSVRC(I), and Pascal VOC2012(P), consisting of 12 classes. 
\textbf{PACS} is a domain adaptation dataset with 9991 images from 4 domains of different
styles: Photo, Artpainting, Cartoon, and Sketch. It is noted that these domains are shared with the same 7 categories.
\textbf{DomainNet} is a challenging domain adaptation benchmark with 0.6 million images of 345 categories of 6 different styles: clipart, infograph, painting, quickdraw, real, and sketch.

\subsubsection{More Experimental Results}
To show the effectiveness of the proposed SIG model, we further consider more compared methods. Experiment results for Office-Home, ImageCLEF, PACS, and DomainNet are shown in Table \ref{tab:office-home}, \ref{tab:imageclef}, \ref{tab:pacs}, and \ref{tab:domainnet}, respectively. Note that We report the mean and the standard deviation of our method over 3 random seeds (i.e. 3,4,5).

\begin{table}[!ht]
    \centering
    \caption{Classification results on the Office-home datasets. We employ ResNet50 as the backbone network. Baseline results are taken from (\cite{kong2022partial}). }
    \setlength{\tabcolsep}{2.8mm}{
    \begin{tabular}{l|c c c c|c}
    \toprule
        Models& Art & Clipart & Product & RealWorld & Average \\ \midrule
        \textbf{Source Only \cite{he2016deep}} & 64.5 (0.68) & 52.3 (0.63) & 77.6 (0.23) & 80.7 (0.81) & 68.8 \\ 
        \textbf{DANN \cite{ganin2015unsupervised}} & 64.2 (0.59) & 58.0 (1.55) & 76.4 (0.47) & 78.8 (0.49) & 69.3 \\ 
        \textbf{DANN+BSP \cite{chen2019transferability}} & 66.1 (0.27) & 61.0 (0.39) & 78.1 (0.31) & 79.9 (0.13) & 71.2 \\ 
        \textbf{DAN \cite{long2015learning}} & 68.2 (0.45) & 57.9 (0.65)  & 78.4 (0.05) & 81.9 (0.35) & 71.6 \\ 
        \textbf{MCD \cite{saito2018maximum}} & 67.8 (0.38) & 59.9 (0.55) & 79.2 (0.61) & 80.9 (0.18) & 71.9 \\ 
        \textbf{DCTN \cite{xu2018deep}} & 66.9 (0.60) & 61.8 (0.46) & 79.2 (0.58) & 77.7 (0.59) & 71.4 \\ 
        \textbf{MIAN-$\gamma$ \cite{park2021information}} & 69.8 (0.35) & \textbf{64.2} (0.68) & 80.8 (0.37) & 81.4 (0.24) & 74.1 \\ 
        \textbf{iMSDA \cite{kong2022partial}} & 75.4 (0.86) & 61.4 (0.73) & 83.5 (0.22) & 84.4 (0.38) & 76.1 \\ \midrule
        \textbf{SIG} & \textbf{76.4} (0.37) & 63.9 (0.34) & \textbf{85.4} (0.36) & \textbf{85.8} (0.22) & \textbf{77.8} \\ \bottomrule
    \end{tabular}}
    \label{tab:office-home}
\end{table}

\begin{table}[!ht]
    \centering
    \caption{Classification results on the ImageCLEF datasets. We employ ResNet50 as the backbone network. Baseline results are taken from (\cite{ren2022multi}).}
    \label{tab:imageclef}
    \begin{tabular}{l|ccc|c}
    \toprule
        \textbf{Mode} & I,C$\rightarrow$P & I,P$\rightarrow$C & P,C$\rightarrow$I & Average \\ \midrule
        \textbf{Source Only \cite{he2016deep}} & 77.2 & 92.3 & 88.1 & 85.8 \\ 
        \textbf{DAN \cite{long2015learning}} & 77.6 & 93.3 & 92.2 & 87.7 \\ 
        \textbf{ADDA \cite{tzeng2017adversarial}} & 76.5 & 94.0 & 93.2 & 87.0 \\ 
        \textbf{DANN \cite{ganin2015unsupervised}} & 77.9 & 93.7 & 91.8 & 87.8 \\ 
        \textbf{D-CORAL \cite{sun2016return}} & 77.1 & 93.6 & 91.7 & 87.5 \\ 
        \textbf{DSBN \cite{chang2019domain}} & 77.7 (0.2) & 94.1 (0.3) & 91.9 (0.1) & 87.9 \\ 
        \textbf{DSAN \cite{zhu2020deep}} & 77.6 (0.2) & 95.1 (0.1) & 91.4 (0.6) & 88.1 \\ 
        \textbf{MFSAN \cite{zhu2019aligning}} & 79.1 & 95.4 & 93.6 & 89.4 \\ 
        \textbf{PTMDA \cite{ren2022multi}} & 79.1 (0.2) & 97.3 (0.3) & 94.1 (0.3) & 90.1 \\ \midrule
        \textbf{SIG} & \textbf{79.3} (0.57)& \textbf{97.3} (0.34) & \textbf{94.3} (0.07) & \textbf{90.3} \\ \bottomrule
    \end{tabular}
\end{table}

\begin{table}[!ht]
    \centering
    \caption{Classification results on the PACS datasets. We employ ResNet18 as the backbone network. Baseline results are taken from (\cite{kong2022partial}). }
    \label{tab:pacs}
    \setlength{\tabcolsep}{2.8mm}{
    \begin{tabular}{l|cccc|c}
    \toprule
        \textbf{Model} & A & C & P & S & Average \\ \midrule
        \textbf{Source Only \cite{he2016deep}} & 74.9 (0.88) & 72.1 & 94.5 & 64.7 (1.53) & 76.7 \\ 
        \textbf{DANN \cite{ganin2015unsupervised}} & 81.9 (1.13) & 77.5 (1.26) & 91.8 (1.21) & 74.6 (1.03) & 81.5 \\ 
        \textbf{MDAN \cite{zhao2018adversarial}} & 79.1 (0.36) & 76.0 (0.73) & 91.4 (0.85) & 72.0 (0.80) & 79.6 \\ 
        \textbf{WBN \cite{mancini2018boosting}} & 89.9 (0.28) & 89.7 (0.56) & 97.4 (0.84) & 58.0 (1.51) & 83.8 \\ 
        \textbf{MCD \cite{saito2018maximum}} & 88.7 (1.01) & 88.9 (1.53) & 96.4 (0.42) & 73.9 (3.94) & 87 \\ 
        \textbf{M3SDA \cite{peng2019moment}} & 89.3 (0.42) & 89.9 (1.00) & 97.3 (0.31) & 76.7 (2.86) & 88.3 \\ 
        \textbf{CMSS \cite{yang2020curriculum}} & 88.6 (0.36) & 90.4 (0.80) & 96.9 (0.27) & 82.0 (0.59) & 89.5 \\ 
        \textbf{LtC-MSDA \cite{wang2020learning}} & 90.1 & 90.4 & 97.2 & 81.5 & 89.8 \\ 
        \textbf{T-SVDNet \cite{li2021t}} & 90.4 & 90.6 & 98.5 & 85.4 & 91.2 \\ 
        \textbf{iMSDA \cite{kong2022partial}} & 93.7 (0.32) & 92.4 (0.23) & 98.4 (0.07) & 89.2 (0.73) & 93.4 \\ \midrule
        \textbf{SIG} & \textbf{94.0} (0.07) & \textbf{93.6} (0.49) & \textbf{98.6} (0.06) & \textbf{89.5} (0.71) & \textbf{93.9} \\ \bottomrule
    \end{tabular}}
\end{table}

\begin{table}[!ht]
    \centering
    \small
    \caption{Classification results on the DomainNet datasets. We employ ResNet101 as the backbone network. Baseline results are taken from (\cite{li2021dynamic} and \cite{wang2022self}).}\label{tab:domainnet}
    \setlength{\tabcolsep}{1.4mm}{
    \begin{tabular}{l|cccccc|c}
    \toprule
        \textbf{Model} & Clipart & Infograph & Painting & Quickdraw & Real & Sketch & Average \\ \midrule
        \textbf{Source Only \cite{he2016deep}} & 52.1 (0.51) & 23.4 (0.28) & 47.6 (0.96) & 13.0 (0.72) & 60.7 (0.23) & 46.5 (0.56) & 40.6 \\ 
        \textbf{ADDA \cite{tzeng2017adversarial}} & 47.5 (0.76) & 11.4 (0.67) & 36.7 (0.53) & 14.7 (0.50) & 49.1 (0.82) & 33.5 (0.49) & 32.2 \\ 
        \textbf{MCD \cite{saito2018maximum}} & 54.3 (0.64) & 22.1 (0.70) & 45.7 (0.63) & 7.6 (0.49) & 58.4 (0.65) & 43.5 (0.57) & 38.5 \\ 
        \textbf{DANN \cite{ganin2015unsupervised}} & 60.6 (0.42) & 25.8 (0.43) & 50.4 (0.51) & 7.70.68) & 62.0 (0.66) & 51.7 (0.19) & 43.0 \\ 
        \textbf{DCTN \cite{xu2018deep}} & 48.6 (0.73) & 23.5 (0.59) & 48.8 (0.63) & 7.2 (0.46) & 53.5 (0.56) & 47.3 (0.47) & 38.2 \\ 
        \textbf{M$^3$SDA-$\beta$ \cite{peng2019moment}} & 58.6 (0.53) & 26.0 (0.89) & 52.3 (0.55) & 6.3 (0.58) & 62.7 (0.51) & 49.5 (0.76) & 42.6 \\ 
        \textbf{ML\_MSDA \cite{li2020mutual}} & 61.4 (0.79) & 26.2 (0.41) & 51.9 (0.20) & 19.1 (0.31) & 57.0 (1.04) & 50.3 (0.67) & 44.3 \\ 
        \textbf{meta-MCD \cite{li2020online}} & 62.8 (0.22) & 21.4 (0.07) & 50.5 (0.08) & 15.5 (0.22) & 64.6 (0.16) & 50.4 (0.12) & 44.2 \\ 
        \textbf{LtC-MSDA \cite{wang2020learning}} & 63.1 (0.5) & 28.7 (0.7) & 56.1 (0.5) & 16.3 (0.5) & 66.1 (0.6) & 53.8 (0.6) & 47.4 \\ 
        \textbf{CMSS \cite{yang2020curriculum}} & 64.2 (0.18) & 28.0 (0.20) & 53.6 (0.39) & 16.9 (0.12) & 63.4 (0.21) & 53.8 (0.35) & 46.5 \\ 
        \textbf{DRT+ST \cite{li2021dynamic}} & 71.0 (0.21) & 31.6 (0.44) & 61.0 (0.32) & 12.3 (0.38) & 71.4 (0.23) & 60.7 (0.31) & 51.3 \\ 
        \textbf{SPS \cite{wang2022self}} & 70.8 & 24.6 & 55.2 & 19.4 & 67.5 & 57.6 & 49.2 \\ 
        \textbf{PFDA \cite{fu2021partial}} & 64.5 & 29.2 & 57.6 & 17.2 & 67.2 & 55.1 & 48.5 \\ \midrule
        \textbf{SIG} & \textbf{72.7} (0.42) & \textbf{32.0} (0.71) & 60.9 (0.87) & \textbf{20.5} (0.71) & \textbf{72.4} (0.14)    & 59.5 (0.70) & 53.0 \\ \bottomrule
    \end{tabular}
     }
\end{table}

\section{Sensitive Analysis of Hyper-parameters}
\label{app:sensitive}

We also consider the sensitive analysis of $\alpha$ and $\beta$, which is shown in Figure \ref{fig:sen_alpha} and \ref{fig:sen_beta}. In detail, we consider different values of $\alpha$ ($\{0.1, 0.3, 0.5, 0.7, 0.9, 1.1, 1.3\}$). According to the experiment results, we find that the model performance is stable with $\alpha$. We also try different values of $\beta$ ($\{1e-5, 3e-5, 5e-5, 7e-5, 9e-5, 1e-4, 5e-4, 1e-3\}$), we find that the model performance is stable in the range of $1e-5\sim5e-4$, but it drop slightly when the value of $\beta$ becomes too large, e.g. $1e-3$.

\begin{figure}[!ht]
\centering
\subfigure[Sensitive results of $\alpha$]{
\label{fig:sen_alpha}
\includegraphics[width=0.45\columnwidth]{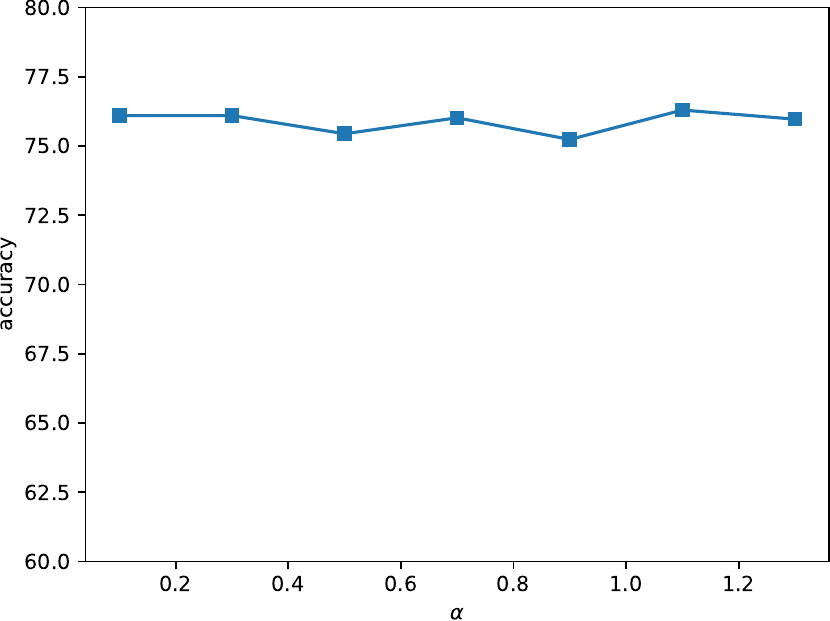}}
\quad
\subfigure[Sensitive results of $\beta$]{
\label{fig:sen_beta}
\includegraphics[width=0.45\columnwidth]{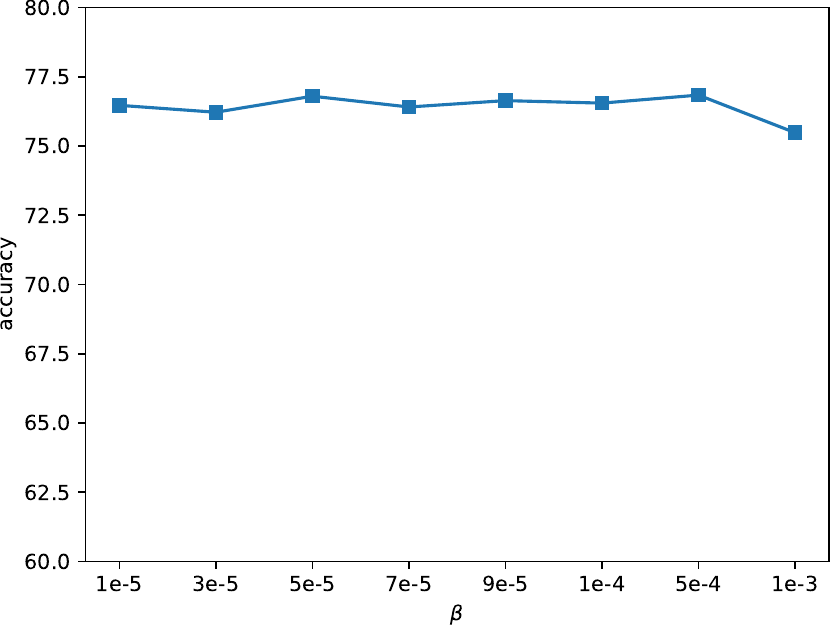}}
\caption{Sensitive analysis of $\alpha$ and $
\beta$ on the $\rightarrow$ Task in Office-Home.}
\label{fig:tsne_sig}
\end{figure}

%

\section{Visualization}
\label{app:visu}
To evaluate the effectiveness of the SIG model qualitatively, we also provide the visualization results in t-SNE as shown in Figure \ref{fig:tsne_sig}. According to the visualization, we can find that our SIG model can generate the features with a more clear class boundary.

\begin{figure}[!ht]
\centering
\subfigure[iMSDA]{
\label{fig:tsne_imsda}
\includegraphics[width=0.45\columnwidth]{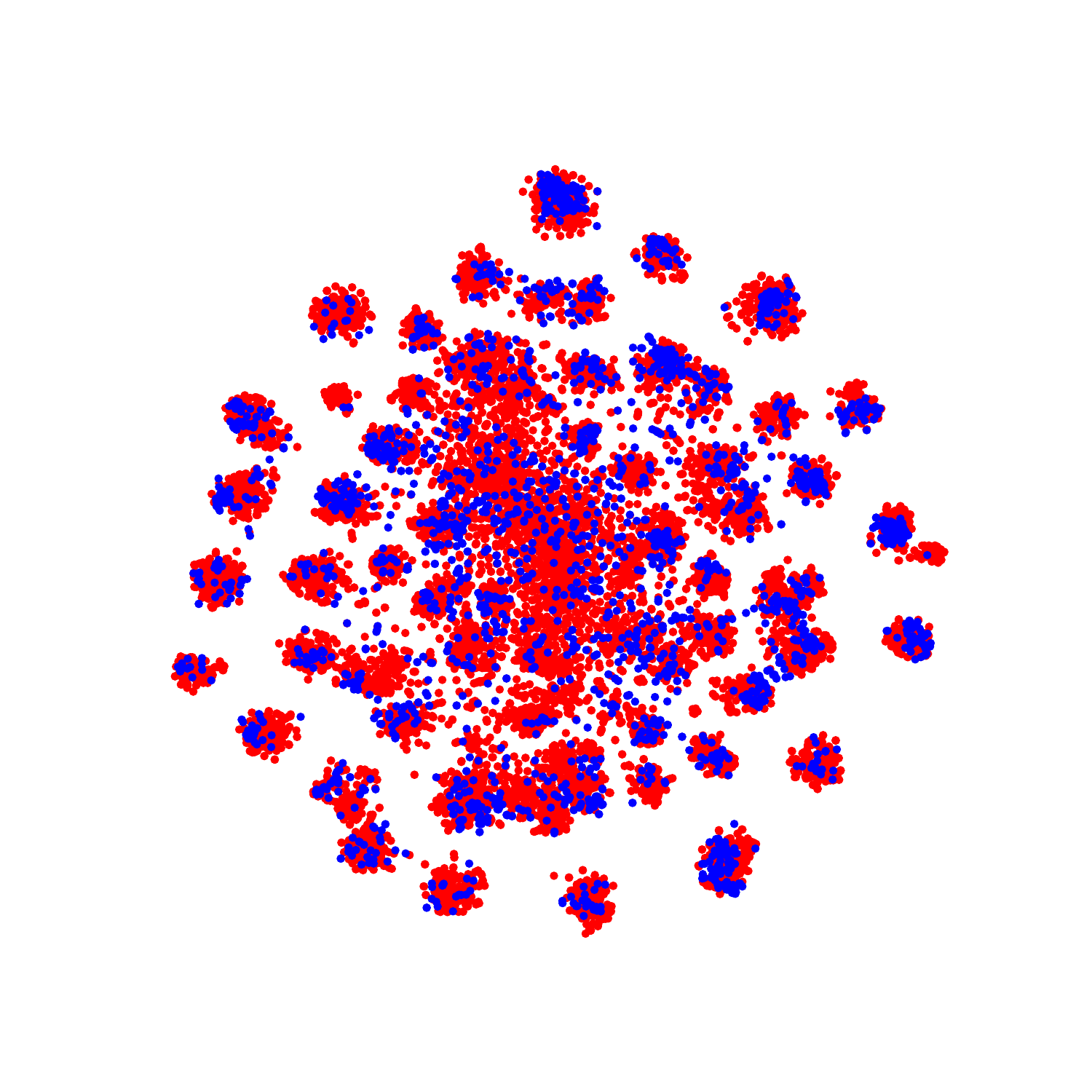}}
\quad
\subfigure[SIG]{
\label{time.2}
\includegraphics[width=0.45\columnwidth]{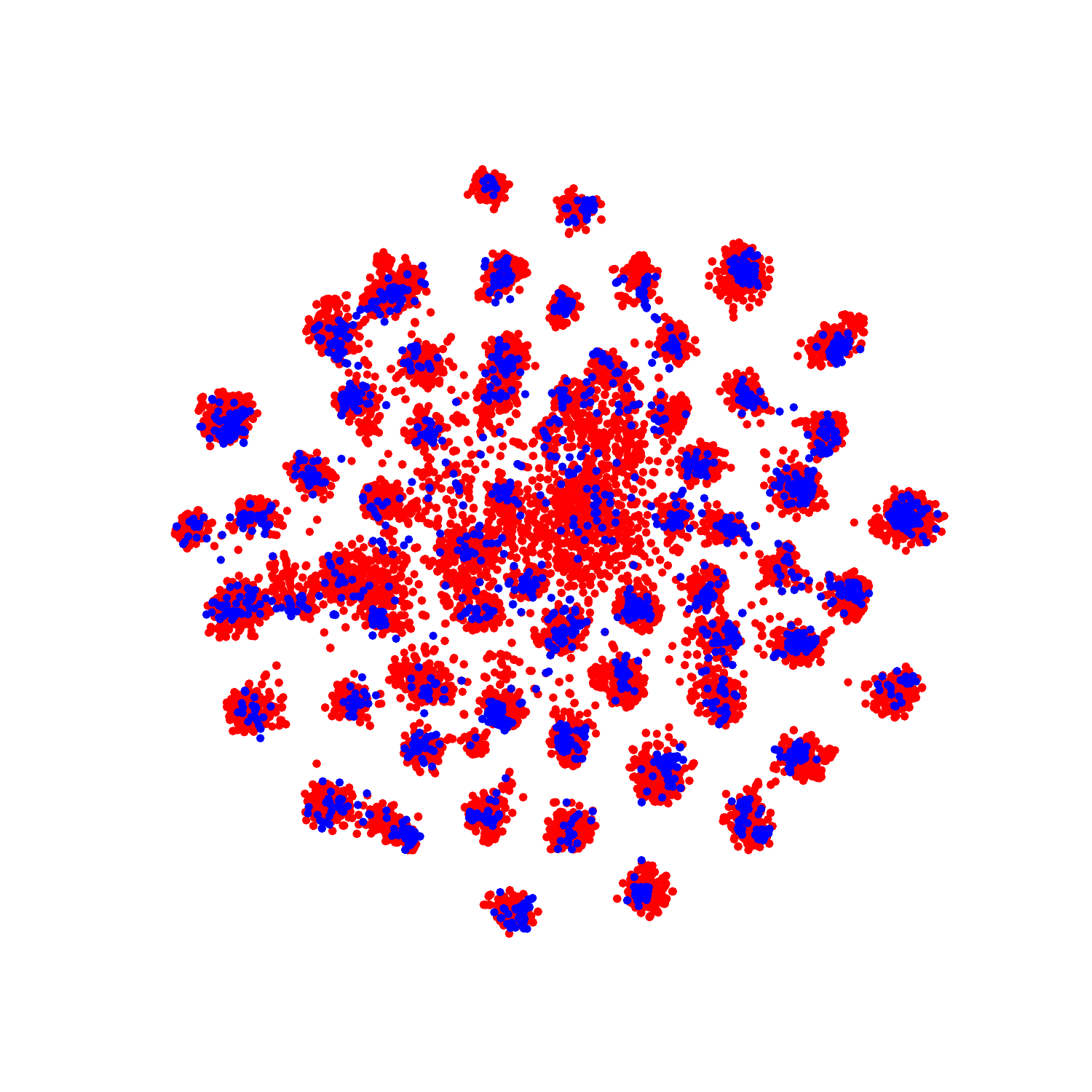}}
\caption{The t-SNE visualizations of learned features on the $\rightarrow$ Art task in Office-Home. Red: source domains, Blue: target domain.}
\label{fig:tsne_sig}
\end{figure}

\section{Related Works}
\label{app:related_works}
\subsection{Domain Adaptation}
Domain adaptation \cite{cai2019learning,zhang2013domain,li2021transferable,kong2022partial,zhang2015multi,zhang2020domain,wen2019bayesian,shui2021aggregating,robey2021model} leverages the knowledge from the labeled source data and unlabeled target data to build a model with ideal generalization. Several researchers solve the challenges of domain adaptation from different perspectives. One of the most conventional directions is to learn the domain-invariant representation \cite{bousmalis2016domain}, which is raised by \cite{ganin2015unsupervised}. Specifically, the key idea of these methods is to extract the domain-invariant representation by aligning the features from different domains. Some researchers \cite{long2017deep} use maximum mean discrepancy (MMD) to realize the domain alignment. Tzeng et.al~\cite{tzeng2014deep} extract the domain-invariant representation by using an adaptation layer and a domain confusion loss. 
Another type of idea assumes that the conditional distributions $P(\bm{z}|y)$ are stable across domains and extract the domain-invariant representation condition on each class \cite{chen2019progressive,chen2019joint,kang2020contrastive}. Specifically, Xie et.al \cite{xie2018learning} minimize the domain discrepancy of inter-class features; Shu et.al \cite{shu2018dirt} consider that the decision boundaries should not cross high-density data regions so they propose the virtual adversarial domain adaptation model. 
Target shift \cite{zhang2013domain,lipton2018detecting,wen2020domain,garg2020unified,roberts2022unsupervised} is also common in domain adaptation, which assumes $p_{\rvy|\rvu}$ varies with different domains. Shui et.al \cite{shui2021aggregating} propose a unified framework to select relevant sources based on the similarity of the conditional distribution. And Remi et.al \cite{tachet2020domain} analyze the generalized label shift and further provide theoretical guarantees on the transfer performance of any classifier.
Recently, several researchers address the domain adaptation problem from the lens of causality \cite{kong2022partial,magliacane2018domain,teshima2020few,chen2021domain,gong2016domain,stojanov2019data}. Zhang et.al~\cite{zhang2013domain} assume that $P(y)$ and $P(\bm{x}|y)$ change independently, and raise the target shift, conditional shift, and generalized target shift assumptions. Cai et.al~\cite{cai2019learning} employ the causal generation process to extract the disentangled semantic representation. Based on the causal analysis, Petar et.al~\cite{stojanov2021domain} find that the domain-invariant should be extracted with the help of domain knowledge, so they propose domain-specific adversarial networks. Despite the outstanding performance of the aforementioned methods, these methods are built on the ad-hoc causal generation process and can not identify the latent variables. In the paper, the proposed \textbf{SIG} method is built on a more general causal generation process and identifies the latent variables with the help of the subspace identification guarantee. 

\subsection{Identification}
To endow more explanation and generalization for the deep generative model, causal representation learning \cite{scholkopf2021toward,kumar2017variational,locatello2019challenging,locatello2019disentangling,zheng2022identifiability,trauble2021disentangled}, which captures the underlying factors and describe the latent generation process, is receiving more and more attention. One of the most classical approaches to learn the causal representation is the independent component analysis (ICA) \cite{hyvarinen2002independent,hyvarinen2013independent,zhang2008minimal,zhang2007kernel,xiemulti,comon1994independent}, in which the generation process is assumed to be a linear mixture function. However, the nonlinear ICA is a challenging task since the latent variables are not identifiable without any extra assumptions on the distribution of latent variables or the generation process \cite{hyvarinen1999nonlinear,zheng2022identifiability,hyvarinen2023identifiability,khemakhem2020ice}. Recently, Aapo et.al \cite{hyvarinen2016unsupervised,hyvarinen2017nonlinear,hyvarinen2019nonlinear,khemakhem2020variational,halva2021disentangling,halva2020hidden} provide the identification theories by introducing auxiliary variables, e.g. domain indexes, time indexes, and class label. These methods usually assume that the latent variables are conditionally independent and follow the exponential families. Recently, Zhang et.al \cite{kong2022partial,xiemulti} break the restriction of exponential families assumption and propose the component-wise identification results for nonlinear ICA with a certain number of auxiliary variables. Following these theoretical results, Yao et.al \cite{yao2022temporally,yao2021learning} recover time-delay latent causal variables and identify their relations from sequential data under the stationary environment and different distribution shifts. Xie et.al \cite{xiemulti} employ the nonlinear ICA to reconstruct the joint distribution of images from different domains; and Kong et.al \cite{kong2022partial} use the component-wise identification results to solve the domain adaptation problem. However, existing identification results heavily rely on a sufficient number of domains and the too-strong monotonic transformation of latent variables, which is hard to satisfy in practice. In this paper, we propose the subspace identification results, which only rely on fewer auxiliary variables compared with component-wise identification and do not rely on any monotonic transformation assumptions.





\end{document}